\documentclass[twoside,11pt]{article}

\usepackage{blindtext}

\usepackage[abbrvbib, preprint]{jmlr2e}

\usepackage{tikz}
\usepackage{tablefootnote}
\usepackage{enumitem}
\usepackage{mathtools}
\usepackage{multirow}
\usepackage{amssymb}
\usepackage{amsmath}

\usepackage{tikz}
\usepackage{graphics}
\usepackage{subcaption}
\usetikzlibrary{decorations.pathreplacing,calligraphy}

\tikzset{
    ncbar angle/.initial=90,
    ncbar/.style={
        to path=(\tikztostart)
        -- ($(\tikztostart)!#1!\pgfkeysvalueof{/tikz/ncbar angle}:(\tikztotarget)$)
        -- ($(\tikztotarget)!($(\tikztostart)!#1!\pgfkeysvalueof{/tikz/ncbar angle}:(\tikztotarget)$)!\pgfkeysvalueof{/tikz/ncbar angle}:(\tikztostart)$)
        -- (\tikztotarget)
    },
    ncbar/.default=0.5cm,
}
\tikzset{round left paren/.style={ncbar=0.5cm,out=120,in=-120}}

\usepackage{algorithm}
\usepackage{algpseudocode}
\algnewcommand{\Input}[1]{%
  \State \textbf{Input:}
  \Statex \hspace*{\algorithmicindent}\parbox[t]{.8\linewidth}{\raggedright #1}
}
\algnewcommand{\Initialize}[1]{%
  \State \textbf{Initialize:}
  \Statex \hspace*{\algorithmicindent}\parbox[t]{.8\linewidth}{\raggedright #1}
}
\algnewcommand{\Inpt}[1]{%
  \State \textbf{Input & Initialize:}
  \Statex \hspace*{\algorithmicindent}\parbox[t]{.8\linewidth}{\raggedright #1}
}
\newcommand\Algphase[1]{%
\vspace*{-.7\baselineskip}\Statex\hspace*{\dimexpr-\algorithmicindent-2pt\relax}\rule{\textwidth}{0.4pt}%
\Statex\hspace*{-\algorithmicindent}\textbf{#1}%
\vspace*{-.7\baselineskip}\Statex\hspace*{\dimexpr-\algorithmicindent-2pt\relax}\rule{\textwidth}{0.4pt}%
}

\hypersetup{ hidelinks }

\makeatletter
\let\@fnsymbol\@arabic
\makeatother

\newcommand{\Kk}{\mathcal{K}}
\newcommand{\Nn}{\mathcal{N}}

\newcommand{\RR}{\mathbb{R}}
\newcommand{\ZZ}{\mathbb{Z}}
\newcommand{\Hom}{\operatorname{Hom}}

\newcommand{\etal}{\textit{et al.}}

\usepackage{lastpage}
\jmlrheading{25}{2024}{1-\pageref{LastPage}}{12/23; Revised
1/24}{12/24}{23-1610}{Jaesun Shin, Eunjoo Jeon, Taewon Cho, Namkyeong Cho, Youngjune Gwon}
\ShortHeadings{LGVR Persistence Diagram}{Shin, Jeon, Cho, Cho, and Gwon}

\firstpageno{1}

\begin{document}

\title{Line Graph Vietoris-Rips Persistence Diagram for Topological Graph Representation Learning}

\author{\name Jaesun Shin \email j1991.shin@samsung.com \\
       \addr Samsung SDS
       \AND
       \name Eunjoo Jeon \email ej85.jeon@samsung.com \\
       \addr Samsung SDS
       \AND
       \name Taewon Cho \email taewon08.cho@samsung.com \\
       \addr Samsung SDS
       \AND
       \name Namkyeong Cho\thanks{Work done while working at Samsung SDS} \email namkyeong.cho@gmail.com \\
       \addr Center for Mathematical Machine Learning and its Applications(CM2LA), Department of Mathematics POSTECH
       \AND
       \name Youngjune Gwon \email gyj.gwon@samsung.com \\
       \addr Samsung SDS}

\editor{Sayan Mukherjee}

\maketitle

\begin{abstract}
While message passing graph neural networks result in informative node embeddings, they may suffer from describing the topological properties of graphs. To this end, node filtration has been widely used as an attempt to obtain the topological information of a graph using persistence diagrams. However, these attempts have faced the problem of losing node embedding information, which in turn prevents them from providing a more expressive graph representation. To tackle this issue, we shift our focus to edge filtration and introduce a novel edge filtration-based persistence diagram, named Topological Edge Diagram (TED), which is mathematically proven to preserve node embedding information as well as contain additional topological information. To implement TED, we propose a neural network based algorithm, named Line Graph Vietoris-Rips (LGVR) Persistence Diagram, that extracts edge information by transforming a graph into its line graph. Through LGVR, we propose two model frameworks that can be applied to any message passing GNNs, and prove that they are strictly more powerful than Weisfeiler-Lehman type colorings. Finally we empirically validate superior performance of our models on several graph classification and regression benchmarks.
\end{abstract}


\section{Introduction}

Recently, message passing graph neural networks and its variants have emerged as an effective method to learn graph representations (\cite{123gnn, dagcn, dsgcn, ccn, graphnorm, graphsage, glimer, review}). Since message passing GNNs are designed to learn node representations, they can extract informative node embeddings by capturing localized information. However, they can hardly capture topological information of the entire graph (\cite{substructure, gfl, breaking}). In this vein, various topological methods have been proposed (\cite{dgcnn, diffpool, asap, gsn}). In particular, node filtration has been widely used to extract topological information of graphs using persistence diagrams (\cite{persistence}), showing superior performance on graph benchmarks (\cite{gfl,togl,pd1, pd2, pd3, pd4}). For example, \cite{gfl} extracted topological information from the persistence diagram of sublevel sets of a node filtration map over node features. Moreover, \cite{togl} proposed a multi-scale version of \cite{gfl} using $k$ node filtration maps. Theoretically, they claim that persistence diagrams based on node filtration map $f_N$ can provide stronger expressivity than the Weisfeiler-Lehman (WL) test (\cite{wl-test}), assuming that the input node features of $f_N$ are all different for each node \cite[Theorem 1, 2]{togl}. However, even the WL test, which is known to be at least as powerful as message passing GNNs (\cite{gin}), cannot assign different features to different nodes. Furthermore, in cases where this assumption does not hold, there exist graphs that cannot be distinguished by node filtration-based persistence diagrams but can be distinguished by the WL test (Figure~\ref{fig:node_edge_diff_2}). In other words, node filtration-based persistence diagram cannot provide more powerful GNNs than the WL test in general.

\begin{figure}[t]
    \centering
    \includegraphics[width=1.0\textwidth]{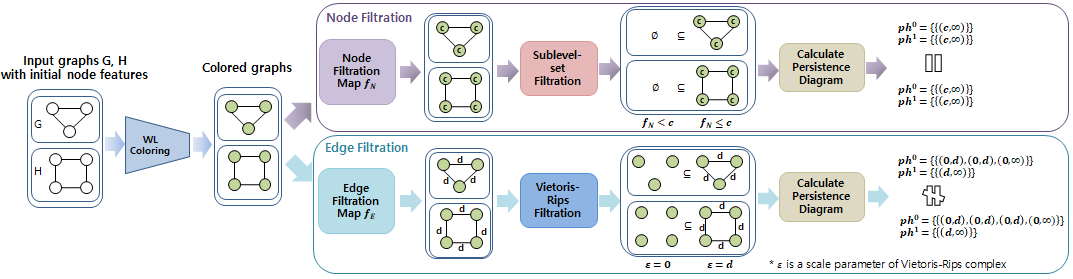}
    \caption{An overview of node filtration and edge filtration of WL coloring. For graphs $G$ and $H$ distinguishable by WL test, the edge filtration-based persistence diagram can also distinguish them while the node filtration-based one cannot.}
    \label{fig:node_edge_diff_2}
\end{figure}

We found out that such topological methods suffer from a loss of coloring information, due to the nature of node filtration, which extracts induced subgraphs (Figure~\ref{fig:node_edge_diff_1}). To address this, we shift our attention to edges. If we can capture information of two nodes in a single edge, we can extract more fruitful information from edge filtration by directly controlling the connectivity of graphs through edges while including all the node information. With this goal in mind, we propose a novel edge filtration-based persistence diagram, \emph{Topologcial Edge Diagram (TED)}, which, to the best of our knowledge, is the first approach using edge-based filtration in topological graph representation learning. Specifically, TED is defined as a persistence diagram of Vietoris-Rips filtration (\cite{hatcher, persistence}), which is the well-known algebraic topological reconstruction technique, whose a set of graph nodes is represented as a point cloud and edge information as distances between points. In contrast to node filtration, we prove that TED can preserve the expressive power of an arbitrary node coloring (Lemma \ref{lem:core}). We further prove that TED can even increase the expressive powers of WL colorings thanks to its additional topological information (Theorem \ref{thm:wl_vr_comparison}). 

Next, we propose a novel neural-network-based algorithm, called \emph{Line Graph Vietoris-Rips (LGVR) Persistence Diagram}, to implement our theoretical foundation. A key challenge is to assign unique features to edges that consist of different node features. To tackle this problem, we construct a map $t_{\phi}$ (Section~\ref{subsection:construction}) that transforms a colored graph into a colored line graph (Definition \ref{defn:colored_line_graph}) through a neural network, which ensures the uniqueness of edge features. Thanks to this, we prove that LGVR has the same expressivity as TED and further analyze its theoretical expressivity. Through LGVR, we propose two types of topological model frameworks, {\rm $\mathcal{C}$-LGVR} and {\rm $\mathcal{C}$-LVGR}$^{+}$, that can be applied to any message passing GNNs $\mathcal{C}$ (Section~\ref{subsection:model_architectures}). From a theoretical perspective, we analyze their theoretical expressivity when $\mathcal{C}$ is either GIN (\cite{gin}) or PPGN (\cite{ppgn}), and prove that our models are strictly more powerful than $\mathcal{C}$ (Corollary~\ref{cor:model}). 

In addition to proposing a theoretical framework, we performed experiments on several real-world datasets, including 7 classification and 12 regression tasks related to bioinformatics, social networks, and chemical compounds, to substantiate the superiority of our topological models (Section~\ref{section:experiments}). We focus on three aspects. First, we test whether our models, {\rm $\mathcal{C}$-LGVR} and {\rm $\mathcal{C}$-LVGR}$^{+}$, which theoretically have more powerful representational power than the message passing GNN $\mathcal{C}$, also show better performances empirically. Our findings demonstrate that our topological models outperform the $\mathcal{C}$ by effectively comprehending diverse graph structures (Table~\ref{tab:classification}, \ref{tab:max_perf_classification} and \ref{tab:regression}). Next, we conducted comparative experiments for our edge filtration-based methodology and the existing node filtration-based approach (\cite{gin}) to experimentally evaluate the representational capabilities between topological methodologies. In this vein, we compare their performances on both classification and regression tasks, and found that our approach shows superior performance compared to node filtration-based approach, which validates the superiority of edge filtration-based approach over node filtration-based one (Table~\ref{tab:classification}, Figure~\ref{fig:perf_improv} and Table~\ref{tab:regression}). Finally, we observed that the performances of GNNs vary significantly based on data splits, even for the same dataset (Table~\ref{tab:classification_stdev}). We speculate that this variation is due to the insufficient utilization of graph information in the training data, resulting from the limited representational powers of existing GNNs. Therefore, we measured the standard deviations of performances based on data split for the message passing GNN $\mathcal{C}$, node filtration-based methodology, and our models. As a result, our models exhibited the lowest standard deviation compared to other methods (Figure~\ref{fig:stdev_fig}). From these results, we confirmed that our topological models, which can well reflect various features in graph representations due to strong expressive powers, enable more stable learning compared to the $\mathcal{C}$ and node filtration-based approach.

The main contributions of this paper can be summarized as follows:
\begin{enumerate}
    \item We introduce a novel edge filtration-based persistence diagram, called \emph{Topological Edge Diagram (TED)}. TED is mathematically proven to preserve node embedding information as well as provide additional topological information. As far as we know, our approach is the first to leverage edge-based filtration in topological graph representation learning.
    \item We propose a novel neural-network-based algorithm, called \emph{Line Graph Vietoris-Rips (LGVR) Persistence Diagram}, to implement our theoretical foundation. We prove that LGVR has the same expressivity as TED and further analyze its theoretical expressivity.
    \item By applying the LGVR, we propose two model frameworks: {\rm $\mathcal{C}$-LGVR} and {\rm $\mathcal{C}$-LGVR}$^{+}$ that can be applied to any message passing GNN $\mathcal{C}$, and theoretically demonstrate the strong expressive powers of our models. 
    \item We perform experiments on several real-world datasets, including classification and regression tasks related to bioinformatics, social networks, and chemical compounds. Through these experiments, we demonstrate that our edge filtration-based models not only have strong experimental representational capabilities but also enable stable learning regardless of data split by encapsulating various graph properties in graph representations. 
\end{enumerate}

This paper is organized as follows. Section~\ref{section:preliminaries} provides a brief explanation of the prerequisite knowledge, including the Weisfeiler-Lehman test, GNN, and basics of persistence homology. We summarize some notations and conventions used throughout this paper in Section~\ref{subsection:notations}. In Section~\ref{section:theoretical_framework}, we introduce our novel edge filtration-based persistence diagram, which we call \emph{Topological Edge Diagram (TED)}, and analyze its theoretical expressive power. In Section~\ref{section:lgvr}, we propose a novel neural-network-based algorithm, named \emph{Line Graph Vietoris-Rips (LGVR) Persistence Diagram}, to implement TED, and analyze its theoretical expressive power. In Section~\ref{section:model_framework}, we propose two model frameworks applying LGVR. Specifically, we first propose a simple mathematical technique that integrates the expressive powers of both coloring information and topological information induced by LGVR. Depending on the application of this integration technique, we propose two topological model frameworks, {\rm $\mathcal{C}$-LGVR} and {\rm $\mathcal{C}$-LVGR}$^{+}$, and analyze their theoretical expressive powers. In Section~\ref{section:experiments}, we conduct experiments by focusing on experimental verification of our models. Section~\ref{section:conclusion} concludes the paper with some future work directions.

\section{Preliminaries} \label{section:preliminaries}

In this section, we briefly recall Weisfeiler-Lehman test, graph neural network, and some basic knowledge in algebraic topology related to persistence homology.

\subsection{(Higher) Weisfeiler-Lehman Test} \label{subsection:higher}

The \emph{Weisfeiler-Lehman} test (WL test) is an algorithm which determines the graph isomorphism problem according to the histogram of colors on the vertices where the colors are iteratively aggregated by those of the neighborhood vertices. Precisely, for a finite graph $G=(V,E)$ with the initial coloring $X_0: V \rightarrow \mathbb{Z}, v \mapsto 1$, the $n$-th coloring $X_n\in\Hom(V,\ZZ)$ is given by $X_n(v):=\sum_{u\in\Nn(v)}X_{n-1}(u)$, where $\Nn(v)$ is a neighbor of $v$ in $G$. Then two graphs $G$ and $G'$ are isomorphic only if their associated $n$-th colorings $X_n$ and $X_n'$ coincide for all $n\geq 1$. To distinguish two graphs of the same order $m=|V|$, it suffices to terminate the algorithm in the $n$-th iteration for some $n=O(m^k)$ \cite{Douglass11}, and it is known that the algorithm is effective for a broad class of graphs \cite{Babai-Kucera79}.

Higher-order variant of WL test, named $k$-WL, has been proposed to improve expressive power and apply color refinements iteratively on vertex tuples instead of single vertices. For a given $G=(V,E)$, the initial coloring is defined using the isomorphism type of each $k$-tuple of $V$. Specifically, two $k$-tuples $(v_{1}, \dots, v_{k})$ and $(w_{1}, \dots, w_{k})$ in $V^{k}$ are assigned the same initial color if and only if for all $i, j \in \{1, \dots, k\}$, (1) $v_{i}=v_{j}$ if and only if $w_{i}=w_{j}$, and (2) $v_{i}$ is adjacent to $v_{j}$ if and only if $w_{i}$ is adjacent to $w_{j}$.

Given this initial coloring, it refines the colorings of $k$-tuples iteratively (similar to the WL test) until the histogram of coloring does not change further. In $k$-WL, the neighborhood of $\mathbf{\nu}=(v_{1},\dots,v_{k})\in V^{k}$ is set to $N_{j}(\mathbf{\nu})=\{(v_1,\dots,v_{j-1}, u,v_{j+1},\dots,v_k)\,|\,u\in V\}$, where $j \in \{1, \dots, k\}$ and $u\in V$. Then the coloring update rules are:
\begin{equation*}
X_{t}(\mathbf{\nu})=HASH(X_{t-1}(\mathbf{\nu}), \, \mathcal{N}(\mathbf{\nu}, \, t-1)),
\end{equation*}
where $\mathcal{N}(\mathbf{\nu}, \, t-1) = \{\{X_{t-1}(\mathbf{\nu}') \,|\, \mathbf{\nu}' \in N_{j}(\mathbf{\nu})\}\} \,|\, j \in \{1, \dots, k\}$. We refer to (\cite{cai,wl-fwl,123gnn}) for several results related to WL and $k$-WL.

\subsection{Graph Neural Network}

Graph neural network (GNN) computes the structure of a graph and its node features to learn a representation vector $h_{v}$ of a vertex $v$. Modern GNNs use spatial methods based on a message passing scheme (\cite{gcn}). In short, the learning process of GNNs iteratively updates the node features from those of the neighboring nodes, which formally associates $(1)$ the representation vectors $h_{v}^{(k)}\in\RR^d$, and $(2)$ the aggregation procedure $h_{v}^{(k)}=\varphi^{k}(h_{v}^{(k-1)}, f^k(\{\{h_{u}^{(k-1)}: u \in \{w \in V \text{ }|\text{ } (v, w) \in E\}\}\}))$, given by an \emph{aggregation function} $f^k$ that operates on multisets and a \emph{combine function} $\varphi^k$. To extract the graph-level representation $h_G$, various pooling methods, also called readout operations, have been proposed to summarize the representation vectors $h_v^{(k)}$ of nodes $v\in V$ (\cite{dgcnn,diffpool,asap,gfl}). In terms of the expressive power of $h_{G}$, it is proven in \cite{gin} that GNNs are as powerful as the WL test under the assumption on the injectivity of $f^{k}$ and $\varphi^{k}$ for each $k$.

\subsection{Simplicial Complexes, Persistence Homology, and Vietoris-Rips Filtration} \label{subsection:persistence_homology}

In this section, we will introduce some basic knowledge in algebraic topology. Readers who are already familiar with algebraic topology may skip this section without hesitation. 

\subsubsection{Simplicial Complexes}

In this subsection, we recall some basics of simplicial complexes. Briefly speaking, a simplex is the simplest geometric object, such as points, line segments, triangles, and their higher-dimensional analogs. Moreover, a simplicial complex is a set of simplices satisfying certain rules. Both are central topological concepts in algebraic topology in order to understand the shape and structure of complex spaces. Formal definitions of both objects are as follows:

\begin{definition}
A \emph{$k$-simplex} is a $k$-dimensional polytope which is the convex hull of affinely independent $k+1$ vertices. Moreover, a \emph{simplicial complex} $\mathcal{K}$ is a set of simplices satisfying the following: (1) every face of a simplex in $\mathcal{K}$ is also in $\mathcal{K}$, and (2) for any $\sigma_1, \sigma_2 \in \mathcal{K}$ such that $\sigma_1 \cap \sigma_2 \neq \emptyset$, $\sigma_1 \cap \sigma_2$ is a face of both $\sigma_1$ and $\sigma_2$. Finally, the \emph{$d$-skeleton} of a simplicial complex $\mathcal{K}$ is the simplicial complex consisting of the set of all simplices in $\mathcal{K}$ of dimension $d$ or less. 
\end{definition}

Finally, a morphism, which is a map preserving structures of mathematical objects, between simplicial complexes can be defined as follows. It is easy to see that simplicial maps can be seen as an extension of graph maps from the perspective of simplicial complexes. 

\begin{definition}
    Let $\mathcal{K}$ and $\mathcal{L}$ be two simplicial complexes. A \emph{simplicial map} $f: \mathcal{K} \rightarrow \mathcal{L}$ is a function from $0$-simplices of $\mathcal{K}$ to $0$-simplices of $\mathcal{L}$ that maps every simplex in $\mathcal{K}$ to a simplex in $\mathcal{L}$. Moreover, a simplicial map $f:\mathcal{K} \rightarrow \mathcal{L}$ is called a \emph{simplicial isomorphism} if it is bijective and its inverse is also a simplicial map. If there exists a simplicial isomorphism between $\mathcal{K}$ and $\mathcal{L}$, we call $\mathcal{K}$ and $\mathcal{L}$ are isomorphic and denote it by $\mathcal{K} \cong \mathcal{L}$
\end{definition}

\subsubsection{Homology and Betti Numbers}

Homology is an abstract way of associating topological or algebraic spaces with a sequence of algebraic objects. In algebraic topology, this allows to encode the topological information of a space through a chain of vector spaces and linear maps. We refer to \cite{hatcher} for interested readers. In general, homology can be defined over an arbitrary field, but for simplicity, we restricted our attention to $\mathbb{Z}_{2}$. Furthermore, we will only deal with the homology classes whose algebraic objects are vector spaces. 

Let $C_{0}, C_{1}, \dots$ be vector spaces over $\mathbb{Z}_{2}$, and let $\partial_{n}: C_{n} \rightarrow C_{n-1}$ be linear maps satisfying $\partial_{n+1} \circ \partial_{n}=0$ for all $n \ge 0$, which we call \emph{boundary operators}. A \emph{chain complex} refers to the sequence 
\begin{equation*}
C_{\bullet}: \cdots \rightarrow C_{n+1} \xrightarrow{\partial_{n+1}} C_{n} \xrightarrow{\partial_{n}} C_{n-1} \rightarrow \cdots \xrightarrow{\partial_{1}} C_{0} \rightarrow 0.
\end{equation*}
Let ${\rm ker}(\partial_{n})=\{x \in C_{n} \text{ $|$ } \partial_{n}(x)=0\}$ be the kernel of $\partial_{n}$, which we call \emph{cycles}, and let ${\rm im}(\partial_{n})=\{y \in C_{n-1} \text{ $|$ } \text{there exists $x \in C_{n}$ such that $\partial_{n}(x)=y$}\}$ be the image of $\partial_{n}$, which we call \emph{boundaries}. Since $\partial_{n+1} \circ \partial_{n}=0$ holds for all $n$, it is clear that ${\rm im}(\partial_{n+1}) \subseteq {\rm ker}(\partial_{n})$ for all $n$. Since both ${\rm im}(\partial_{n+1})$ and ${\rm ker}(\partial_{n})$ are vector spaces, we may form the quotient vector space 
\begin{equation*}
H_{n}(C_{\bullet}):={\rm ker}(\partial_{n}) / {\rm im}(\partial_{n+1})
\end{equation*}
for all $n \ge 0$. We call $H_{n}(C_{\bullet})$ the \emph{$n$-th homology (group)} of $C_{\bullet}$, and the elements of $H_{n}(C_{\bullet})$ are called \emph{homology classes}. Moreover, the dimension of $H_{n}(C_{\bullet})$ as a vector space, denoted by $\beta_{k}(C_{\bullet})$, is called the \emph{$n$-th Betti number} of $C_{\bullet}$. 

\subsubsection{Simplicial Homology}
We will introduce \emph{simplicial homology}, a type of homology that is widely used in algebraic topology and tailored to our purposes. As the name suggests, simplicial homology is derived from simplicial complexes. 

Given a simplicial complex $\mathcal{K}$, we define a chain complex $C_{\bullet}(\mathcal{K})$ of $\mathbb{Z}_{2}$-vector spaces in the following way: Let $\mathcal{K}_{n}=\{\sigma_{1},\dots, \sigma_{k}\}$ be the $n$-skeleton of $\mathcal{K}$. Then we define $C_{n}(\mathcal{K})$ to be the vector space over $\mathbb{Z}_{2}$ with $\mathcal{K}_{n}$ as a basis. Moreover, for each $n$-simplex $\sigma_{i}=[v_{0},\dots,v_{n}] \in C_{n}(\mathcal{K})$, the boundary operator $\partial_{n}:C_{n}(\mathcal{K}) \rightarrow C_{n-1}(\mathcal{K})$ is defined by
\begin{equation*}
\partial_{n}(\sigma_{i}):= \sum_{j=0}^{n}[v_{0}, \dots, \hat{v_{j}},\dots, v_{n}],
\end{equation*}
where $\hat{v_{j}}$ means that $v_{j}$ is omitted. By extending linearly to all of $C_{n}(\mathcal{K})$, the boundary operator $\partial_{n}$ can be defined on $C_{n}(\mathcal{K})$. 

It is easy to see that the boundary operator $\partial_{n}$ satisfies $\partial_{n+1} \circ \partial_{n}=0$. Hence vector spaces $C_{n}(\mathcal{K})$ with the boundary operators $\partial_{n}$ form a chain complex. This allows us to form the homologies $H_{n}(C_{\bullet}(\mathcal{K}))$ of $C_{\bullet}(\mathcal{K})$, which we call \emph{simplicial homology}. 

\subsubsection{Persistence Homology and Diagrams}
Given a point cloud $P$ sampled from the unknown manifold $M$, how can we determine the topological characteristics of $M$ from $P$? The simplest way is to generate a suitable manifold $K$ from $P$ and compute its homology. However, homology is very sensitive to small changes, so there is a problem that the topological characteristics of $K$ can be very different from those of $M$. \emph{Persistence homology} (\cite{persistence}) addresses this problem by incorporating the scale $\varepsilon$, which varies from $0$ to $\infty$ , into homology computations. As $\varepsilon$ increases, the topological characteristics of a manifold induced by $\varepsilon$ can vary: some topological information born at some $\epsilon_0$ and die $\epsilon_1$. Informally, the idea of persistence homology is to track all the birth and death of topological information with scale $\varepsilon$. 

A persistence homology essentially tracks the evolution of homology classes in a filtration of simplicial complexes $\mathcal{K}$. Once a simplicial complex $\Kk$ admits a filtration 
\begin{equation*}
\Kk^\bullet: \emptyset=\Kk^{-\infty} \subseteq \cdots \subseteq \Kk^{i} \subseteq \cdots \subseteq \Kk^{j} \subseteq \cdots \subseteq \Kk^{\infty}=\Kk,
\end{equation*}
then each inclusion $f_{i,j}: \mathcal{K}^{i} \hookrightarrow \mathcal{K}^{j}$ is a simplicial map so that it induces a linear map $H_{n}(f_{i,j})$ between the homologies of $H_{n}(\mathcal{K}^{i})$ and $H_{n}(\mathcal{K}^{j})$. Such indices $i$ and $j$ are referred to as 'time' in persistence homology theory. Fix a non-negative integer $n \ge 0$. For any $i < j$, one can see whether a homology class in $H_{n}(\mathcal{K}^{i})$ are mapped to the same homology class in $H_{n}(\mathcal{K}^{j})$ by $H_{n}(f_{i,j})$. If this happens, such a homology class is said to \emph{persist from time $i$ to $j$}. If not, such a class is said to have \emph{died at some time between $i$ and $j$}. If a homology class first appears at time $i$ and disappears at time $j$, then we say that this class is \emph{born at time $i$ and dies at time $j$}, and appends $(i, j)$ as an element of \emph{$n$-th persistence homology}. By tracking all homology classes of $\{H_{n}(\mathcal{K}^{t})\}_{t \in \mathbb{R}}$ and appending them as elements of $n$-th persistence homology, the $n$-th persistence homology can be seen as a multi-set of birth and death tuples. Now, regard the $n$-th persistence homology as a multi-set of points in $\mathbb{R}^{2}$. Then we call such a multi-set \emph{$n$-th persistence diagram}. 

Informally, $n$-th persistence homology (or diagram) tracks different topological features depending on n. For example, $0$-th persistence homology tracks the birth and death of connected components, while $1$-th persistence homology tracks those of circular holes. The general $n$-th persistence homology has information about the birth and death of $n$-dimensional holes. Through persistence homology information for each n, we can understand the characteristics of a given topological object.

\subsubsection{Vietoris-Rips Complex and Filtration}
Given a point cloud $P$ with a distance matrix $M$ and a scale $\varepsilon$, the \emph{Vietoris-Rips complex} of $\varepsilon$ is a type of simplicial complex constructed from $P$ and $M$ whose simplices are formed by connecting points in $P$ that are within a certain distance $\varepsilon$ of each other. In particular, given two scales $\varepsilon_1 < \varepsilon_2$, the Vietoris-Rips complex of $\varepsilon_1$ is contained in that of $\varepsilon_2$. Thus, by adjusting a scale $\varepsilon$, we can define a filtration of simplicial complexes, called \emph{Vietoris-Rips filtration}. This makes it possible to analyze the persistence homologies of a point cloud. Here we will provide their definitions below.

\begin{definition} \label{defn:vr_cpx}
    Let $P$ be a finite point cloud, and let $M$ be a non-negative symmetric matrix of size $|P| \times |P|$ with zero diagonals. The \emph{Vietoris-Rips} complex of $(P, M)$ with a scale $\varepsilon \in \mathbb{R}_{\ge 0}$, denoted as ${\rm VR}^{\varepsilon}(P, M)$, has one $t$-simplex per $(t+1)$-tuple of points $(u_0, \dots, u_t)$ of $P$ such that $M(u_i, u_j) \le \varepsilon$ for all $i,j =0, \dots, t$. Moreover, the \emph{Vietoris-Rips filtration} of $(P, M)$ is the indexed family ${\rm VR}(P,M)=\{ {\rm VR}^{\varepsilon}(P, M)\}_{\varepsilon \in \mathbb{R}}$. Finally, we denote the $k$-skeleton of ${\rm VR}^{\varepsilon}(P, M)$ as ${\rm VR}_{k}^{\varepsilon}(P, M)$, and let ${\rm VR}_{k}(P,M)=\{ {\rm VR}_{k}^{\varepsilon}(P, M)\}_{\varepsilon \in \mathbb{R}}$ for $k \in \mathbb{Z}_{\ge 0}$. 
\end{definition}

Finally, we will provide a short remark regarding the matrix $M$ used in Definition~\ref{defn:vr_cpx}. In general, when defining the Vietoris-Rips complex, the matrix $M$ is constructed based on the distance between two points according to some metric. However, since the Vietoris-Rips complex can be defined based solely on the pairwise distances between points, we follow a general definition (Definition~\ref{defn:vr_cpx}) with minimal conditions on the matrix $M$, independent of a metric. In other words, the matrix $M$ in Definition~\ref{defn:vr_cpx} may not satisfy properties of metrics such as triangle inequality. For example, $M$ may not satisfy the triangle inequality, that is, $M_{i,j} + M_{j,k} \nleq M_{i,k}$ for some $i$, $j$, and $k$. However, since the existence of a metric does not affect the theoretical results that will be developed later in this paper, we will use the general version of the Vietoris-Rips complex based on such a matrix $M$.

\section{Notations and conventions} \label{subsection:notations}

In this section, we summarize some notations used throughout this paper. First of all, we use the following notations, which are commonly used, to distinguish between a set and a multi-set: we denote a set by $\{\cdots\}$ and a multi-set by $\{\{\cdots\}\}$. Moreover, the WL test refers to the $1$-WL test, otherwise specified. 

Let $(\mathcal{G}, \mathcal{C})$ be a space of graphs with node coloring $\mathcal{C}$, and let $\chi_{\mathcal{G}} \subseteq \mathbb{R}^{N}$ (or simply $\chi$) be a space of node features of $(\mathcal{G}, \mathcal{C})$ containing $(0,\dots,0)$, where $N \in \mathbb{N}$. For a graph $G \in \mathcal{G}$, $V(G)$ denotes a set of nodes in $G$, $E(G)$ denotes a set of edges in $G$, and $E^{\mathcal{C}}(G)$ denotes a multi-set of colored edges in $G$, that is, $E^{\mathcal{C}}(G)=\{\{ \text{ }\{\{\mathcal{C}(u), \mathcal{C}(v)\}\} \text{ $|$ } \{\{u, v\}\} \in E(G)\}\}$. 

Finally, since we frequently use the notation in Definition~\ref{defn:vr_cpx}, we summarize it here again. Given a finite point cloud $P$ and a non-negative symmetric matrix $M$ of size $|P| \times |P|$ with zero diagonals, we denote the $k$-skeleton of ${\rm VR}^{\varepsilon}(P, M)$ as ${\rm VR}_{k}^{\varepsilon}(P, M)$, and put ${\rm VR}_{k}(P,M)=\{ {\rm VR}_{k}^{\varepsilon}(P, M)\}_{\varepsilon \in \mathbb{R}}$ for $k \in \mathbb{Z}_{\ge 0}$.

\section{Theoretical Framework: Topological Edge Diagram} \label{section:theoretical_framework}

In this section, we introduce a novel edge filtration-based persistence diagram, which we call \emph{Topological Edge Diagram}. To begin with, we introduce the \emph{edge filtration} map, and then define the topological edge diagram by leveraging it. In the end, we will wrap up this section by analyzing its theoretical expressive power.

\begin{figure}[t]
    \centering
    \includegraphics[width=0.9\textwidth]{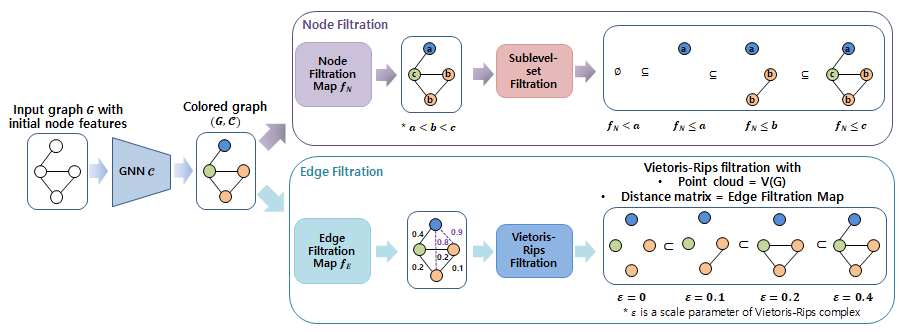}
    \caption{The difference in subgraph construction between two types of filtrations: node filtration and edge filtration.}
    \label{fig:node_edge_diff_1}
\end{figure}

\subsection{Definition of Topological Edge Diagram (TED)}

To introduce our persistence diagram, we first introduce the \emph{edge filtration} (map) which is essential in developing our theoretical framework. To grasp the distinction between node (\cite{gfl, togl}) and edge filtration, see Figure \ref{fig:node_edge_diff_1}. 

\begin{definition} \label{defn:edge_filtration}
    Let $G \in \mathcal{G}$ be a graph, and let $\mathcal{C}$ be a node coloring of $\mathcal{G}$. 
    \begin{enumerate}
        \item An \emph{edge filtration (map)} $ef^{\mathcal{C}}$ of $\mathcal{G}$ with respect to $\mathcal{C}$ is defined to be a positive real-valued function $ef^{\mathcal{C}}:\bigcup_{G \in \mathcal{G}} E^{\mathcal{C}}(G) \rightarrow \mathbb{R}_{>0}$ such that $\sup{\{ef^{\mathcal{C}}(x) \text{ }|\text{ } x \in \bigcup_{G \in \mathcal{G}} E^{\mathcal{C}}(G)\}} < \infty$, where $E^{\mathcal{C}}(G)=\{\{ \text{ }\{\{\mathcal{C}(u), \mathcal{C}(v)\}\} \text{ $|$ } \{\{u, v\}\} \in E(G)\}\}$ is a multi-set of colored edges.
        \item For an edge filtration $ef^{\mathcal{C}}$ and $i \ge 0$, ${\rm ph}_{\rm VR}^{i}(G, ef^{\mathcal{C}})$ is the $i$-th persistence diagram of $1$-skeleton of Vietoris-Rips filtration with point cloud $V(G)$ whose distance matrix $M(G)$ is defined as follows: for any $i, j=1, \dots, |V(G)|$ corresponding to nodes $u_i, u_j \in V(G)$, 
        \begin{equation*}
            M(G)_{i,j}=
            \begin{cases}
            ef^{\mathcal{C}}(\{\{\mathcal{C}(u_i), \mathcal{C}(u_j)\}\}) & \text{if } \{\{\mathcal{C}(u_i), \mathcal{C}(u_j)\}\} \in E^{\mathcal{C}}(G), \\
            \infty & \text{otherwise }
            \end{cases}
        \end{equation*}
    \end{enumerate}
\end{definition}


\begin{remark} \label{rmk:natural_choice}
In this remark, we will explain why the $1$-skeleton of the Vietoris-Rips complex of $V(G)$ is a natural choice when defining the persistent homology ${\rm ph}_{\rm VR}^{i}(G, ef^{\mathcal{C}})$ of the graph $G$ (Definition~\ref{defn:edge_filtration}-2). Before going on, we briefly explain the conditions that the filtration $\{X_{\varepsilon}\}_{\varepsilon>0}$ of a simplicial complex $X$ must satisfy to define persistent homology: given a simplicial complex $X$, the filtration $\{X_{\varepsilon}\}_{\varepsilon>0}$ to define the persistence homology of $X$ should satisfy the following condition: for any $\varepsilon>0$, $X_{\varepsilon} \hookrightarrow X$ as simplicial complexes and $\lim_{\varepsilon \rightarrow \infty}X_{\varepsilon}=X$. 

In this perspective, we see how the dimension of the clique complex of $G$ is related when defining its Vietoris-Rips filtration. More precisely, let ${\rm VR}^{\varepsilon}_{d}(V(G), M(G))$ be the $d$-skeleton of Vietoris-Rips complex of scale $\varepsilon >0$ with point cloud $V(G)$ whose distance matrix is $M(G)$ (Definition~\ref{defn:edge_filtration}). Since a graph $G \in \mathcal{G}$ is $1$-simplicial complex, it is easy to see the following:
\begin{enumerate}
    \item for any $\varepsilon>0$, ${\rm VR}^{\varepsilon}_{1}(V(G), M(G)) \hookrightarrow G$ as simplicial complexes,
    \item $\lim_{\varepsilon \rightarrow \infty} {\rm VR}^{\varepsilon}_{1}(V(G), M(G)) = G$, and
    \item for any $d \ge 2$, there exists $\varepsilon>0$ such that ${\rm VR}^{\varepsilon}_{d}(V(G), M(G)) \not\hookrightarrow G$ as simplicial complexes.
\end{enumerate}
In other words, $1$-skeleton is the most natural choice for defining the persistence homology ${\rm ph}_{\rm VR}^{i}(G, ef^{\mathcal{C}})$ of the graph. 
\end{remark}

Note that the edge filtration induces a novel persistence diagram by using the Vietoris-Rips filtration of graph. In particular, when the edge filtration is injective, we call this persistence diagram, \emph{Topological Edge Diagram (TED)}. 

\begin{definition} \label{defn:top_edge_filt}
    Given a graph $G \in \mathcal{G}$ and an injective edge filtration $ef^{\mathcal{C}}$ with a node coloring $\mathcal{C}$, a \emph{topological edge diagram (TED) of $(G, ef^{\mathcal{C}})$}, denoted by ${\rm TED}(G, ef^{\mathcal{C}})$, is defined to be the tuple of multi-sets $({\rm ph}_{\rm VR}^{0}(G, ef^{\mathcal{C}}),\text{ } {\rm ph}_{\rm VR}^{1}(G, ef^{\mathcal{C}}))$.
\end{definition}

\subsection{Theoretical Expressivity of TED}

It is clear that an injective edge filtration $ef^{\mathcal{C}}$ of the node coloring $\mathcal{C}$ includes all the pairwise node coloring information. Hence our question is whether TED also contains all of the node coloring information. For theoretical clarity, we first propose a weak assumption about the node coloring $\mathcal{C}$, called the \emph{Degree Assumption}. 

\begin{definition}[Degree Assumption]
    Let $\mathcal{C}$ be a node coloring of $\mathcal{G}$, that is, $\mathcal{C}: \bigcup_{G \in \mathcal{G}} V(G) \rightarrow \chi$. We say that $C$ satisfies the \emph{degree assumption} if the following holds: for any $u, v \in \bigcup_{G \in \mathcal{G}}V(G)$, if $\mathcal{C}(u)=\mathcal{C}(v)$, then ${\rm deg}(u)={\rm deg}(v)$. 
\end{definition}

We remark that the degree assumption is a very weak assumption related to node coloring, and almost all node colorings we know satisfy this assumption. For example, any node colorings of all message passing GNNs with arbitrary initial node colorings satisfy the Degree Assumption. Furthermore, the WL test, which is a classical graph isomorphism test, satisfies the Degree Assumption as demonstrated in the following example.

Now, we are ready to show the strong expressive power of TED. Specifically, our next result shows that TED can incorporate all of the node coloring information if any graph $G \in \mathcal{G}$ has no isolated nodes.


\begin{lemma} \label{lem:core}
    Let $G, H \in \mathcal{G}$ be graphs with no isolated nodes, and let $\mathcal{C}$ be a node coloring of $\mathcal{G}$ satisfying the degree assumption. If $\{\{\mathcal{C}(u) \text{ $|$ } u \in V(G)\}\} \neq \{\{\mathcal{C}(v) \text{ $|$ } v \in V(H)\}\}$, then ${\rm TED}(G, ef^{\mathcal{C}}) \neq {\rm TED}(H, ef^{\mathcal{C}})$ for any injective edge filtration $ef^{\mathcal{C}}$ and any $k \ge 1$.
\end{lemma}

\begin{proof}
    Since $\{\{\mathcal{C}(u) \text{ }|\text{ } u \in V(G)\}\} \neq \{\{\mathcal{C}(v) \text{ }|\text{ } v \in V(H)\}\}$, there are three cases: 
    \begin{enumerate}
        \item $|V(G)| \neq |V(H)|$,
        \item $|V(G)|=|V(H)|$ and there exists $u \in V(G)$ such that $\mathcal{C}(u) \notin \{\{\mathcal{C}(v) \text{ }|\text{ } v \in V(H)\}\}$,
        \item $|V(G)|=|V(H)|$ and $\{\mathcal{C}(u) \text{ }|\text{ } u \in V(G)\} = \{\mathcal{C}(v) \text{ }|\text{ } v \in V(H)\}$ but $\{\{\mathcal{C}(u) \text{ }|\text{ } u \in V(G)\}\} \neq \{\{\mathcal{C}(v) \text{ }|\text{ } v \in V(H)\}\}$. 
    \end{enumerate}
    Fix an injective edge filtration $ef^{\mathcal{C}}$. First of all, suppose the case $(1)$. Since an edge filtration $ef^{\mathcal{C}}$ is positive, each element in ${\rm ph}_{\rm VR}^{0}(G, ef^{\mathcal{C}})$ should be of the form $(0, t)$ for some $t \in (0, \infty]$. So the number of elements in the multi-set ${\rm ph}_{\rm VR}^{0}(G, ef^{\mathcal{C}})$ is equal to the number of nodes in $G$. Hence $|V(G)| \neq |V(H)|$ implies that the cardinalities of ${\rm ph}_{\rm VR}^{0}(G, ef^{\mathcal{C}})$ and ${\rm ph}_{\rm VR}^{0}(H, ef^{\mathcal{C}})$ are not equal as multi-sets so that ${\rm ph}_{\rm VR}^{0}(G, ef^{\mathcal{C}}) \neq {\rm ph}_{\rm VR}^{0}(H, ef^{\mathcal{C}})$. 

    Next, suppose the case $(2)$. Let $u \in V(G)$ be a node satisfying $\mathcal{C}(u) \notin \{\{\mathcal{C}(v) \text{ }|\text{ } v \in V(H)\}\}$. Since $G$ has no isolated nodes, there exists an edge $e_G \in E(G)$ whose one of the end points is $u$. Since $\mathcal{C}(u) \notin \{\{\mathcal{C}(v) \text{ }|\text{ } v \in V(H)\}\}$, it is easy to see from the injectivity of $ef^{\mathcal{C}}$ that there does not exist an edge $e \in E(H)$ such that $ef^{\mathcal{C}}(e)=ef^{\mathcal{C}}(e_G)$. 
    
    Now we claim that either $(0, ef^{\mathcal{C}}(e_G)) \in {\rm ph}_{\rm VR}^{0}(G, ef^{\mathcal{C}})$ or $(ef^{\mathcal{C}}(e_G), \infty) \in {\rm ph}_{\rm VR}^{1}(G, ef^{\mathcal{C}})$ holds but neither $(0, ef^{\mathcal{C}}(e_G)) \in {\rm ph}_{\rm VR}^{0}(H, ef^{\mathcal{C}})$ nor $(ef^{\mathcal{C}}(e_G), \infty) \in {\rm ph}_{\rm VR}^{1}(H, ef^{\mathcal{C}})$ holds. Note that it is clear that both $(0, ef^{\mathcal{C}}(e_G)) \notin {\rm ph}_{\rm VR}^{0}(H, ef^{\mathcal{C}})$ and $(ef^{\mathcal{C}}(e_G), \infty) \notin {\rm ph}_{\rm VR}^{1}(H, ef^{\mathcal{C}})$ hold since there is no edge $e \in E(H)$ satisfying $ef^{\mathcal{C}}(e)=ef^{\mathcal{C}}(e_G)$. Hence we need to show the first part of the claim. Note that the birth of a colored edge $e=\{\{\mathcal{C}(u), \mathcal{C}(v)\}\} \in E^{\mathcal{C}}(G)$ corresponds to either of the following two cases: 
    \begin{enumerate}[label=(\roman*)]
        \item there was no path between $u$ and $v$ before the birth of $e$, or
        \item there was a path between $u$ and $v$ before the birth of $e$.
    \end{enumerate}
    In case $(i)$, the birth of $e$ leads to the death of a connected component, which corresponds to $(0, ef^{\mathcal{C}}(e_G)) \in {\rm ph}_{\rm VR}^{0}(G, ef^{\mathcal{C}})$. Moreover, the second case $(ii)$ implies that $e$ creates a new cycle so that it corresponds to $(ef^{\mathcal{C}}(e_G), \infty) \in {\rm ph}_{\rm VR}^{1}(G, ef^{\mathcal{C}})$ since ${\rm ph}_{\rm VR}^{i}(G, ef^{\mathcal{C}})$ is computed from the $1$-skeleton of Vietoris-Rips complex with point cloud $V(G)$. This proves our claim. In other words, our claim shows that there is an element that is contained in ${\rm ph}_{\rm VR}^{0}(G, ef^{\mathcal{C}}) \cup {\rm ph}_{\rm VR}^{1}(G, ef^{\mathcal{C}})$ but not in ${\rm ph}_{\rm VR}^{0}(H, ef^{\mathcal{C}}) \cup {\rm ph}_{\rm VR}^{1}(H, ef^{\mathcal{C}})$, we conclude that $({\rm ph}_{\rm VR}^{0}(G, ef^{\mathcal{C}}), \text{ } {\rm ph}_{\rm VR}^{1}(G, ef^{\mathcal{C}})) \neq ({\rm ph}_{\rm VR}^{0}(H, ef^{\mathcal{C}}), \text{ } {\rm ph}_{\rm VR}^{1}(H, ef^{\mathcal{C}}))$. 

    Finally, it remains to show the conclusion holds for the case $(3)$. Since $\{\mathcal{C}(u) \text{ $|$ } u \in V(G)\}=\{\mathcal{C}(v) \text{ $|$ } v \in V(H)\}$, there exists an indexed coloring set $\{C_{i}\}_{i\in I}$ with an index set $I$ such that for each $i \in I$, there exists $u \in V(G)$ and $v \in V(H)$ satisfying $\mathcal{C}(u)=\mathcal{C}(v)=C_{i}$. Towards contradiction, suppose not, that is, ${\rm ph}_{\rm VR}^{k}(G, ef^{\mathcal{C}}) = {\rm ph}_{\rm VR}^{k}(H, ef^{\mathcal{C}})$ for any $k=0, 1$. Again, by the claim in the proof of case $(2)$, this implies that $E^{\mathcal{C}}(G)=E^{\mathcal{C}}(H)$. Hence, for all $i \in I$, we have 
    \begin{equation} \label{eq:same_deg}
    \sum_{\substack{u\in V(G), \\ \mathcal{C}(u)=C_i}} {\rm deg}(u) = \sum_{\substack{v\in V(H), \\ \mathcal{C}(v)=C_i}} {\rm deg}(v).
    \end{equation}
    However, since the multi-set of node features of $G$ and $H$ are not the same, that is, $\{\{ \mathcal{C}(u) \text{ $|$ } u \in V(G)\}\} \neq \{\{ \mathcal{C}(v) \text{ $|$ } v \in V(H)\}\}$, there exists an index $i_0 \in I$ such that 
    \begin{equation*}
    |\{u \in V(G) \text{ $|$ } \mathcal{C}(u)=C_{i_{0}} \}| \neq |\{v \in V(H) \text{ $|$ } \mathcal{C}(v)=C_{i_{0}} \}|.
    \end{equation*}
    By the definition of $I$, both $\{u \in V(G) \text{ $|$ } \mathcal{C}(u)=C_{i_{0}} \}$ and $\{v \in V(H) \text{ $|$ } \mathcal{C}(v)=C_{i_{0}} \}$ are both non-empty, thus we put $\{u \in V(G) \text{ $|$ } \mathcal{C}(u)=C_{i_{0}} \} = \{u_1, \dots, u_g\}$ and $\{v \in V(H) \text{ $|$ } \mathcal{C}(v)=C_{i_{0}} \} = \{v_1, \dots, v_h\}$ with $g \neq h$. Since all $u_i$ and $v_j$ have the same coloring, the degree assumption of $\mathcal{C}$ gives that 
    \begin{equation*}
    {\rm deg}(u_1) = \cdots = {\rm deg}(u_g) = {\rm deg}(v_1) = \cdots = {\rm deg}(v_h).
    \end{equation*}
    Moreover, such a degree value is positive since $G$ and $H$ have no isolated nodes. Hence we have 
    \begin{equation*}
    \sum_{\substack{u \in V(G), \\ \mathcal{C}(u)=C_{i_0}}} {\rm deg}(u) = g \cdot {\rm deg}(u_1) \neq h \cdot deg(v_1) = \sum_{\substack{v \in V(H), \\ \mathcal{C}(v)=C_{i_0}}} {\rm deg}(v),
    \end{equation*}
    which contradicts to Equation \ref{eq:same_deg}. Thus 
    \begin{equation*}
        ({\rm ph}_{\rm VR}^{0}(G, ef^{\mathcal{C}}), \text{ } {\rm ph}_{\rm VR}^{1}(G, ef^{\mathcal{C}})) \neq ({\rm ph}_{\rm VR}^{0}(H, ef^{\mathcal{C}}), \text{ } {\rm ph}_{\rm VR}^{1}(H, ef^{\mathcal{C}})),
    \end{equation*}
    which concludes the proof.     
\end{proof}


Lemma \ref{lem:core} presents a general result applicable to any node coloring $\mathcal{C}$. Moving forward, we will narrow our focus to WL coloring. Since WL coloring satisfies the degree assumption, Lemma \ref{lem:core} implies that ${\rm TED}(\cdot, ef^{\mathcal{C}})$ is at least as powerful as the WL test for any edge filtration $ef^{\mathcal{C}}$, assuming non-isolated nodes. Our next theorem, which is the core of our theoretical framework, shows that a stronger result holds for WL coloring: TED is strictly more powerful than the WL test without assuming non-isolated nodes. 


\begin{theorem} \label{thm:wl_vr_comparison}
    Let $G, H \in \mathcal{G}$, and let $\mathcal{C}$ be the stable WL coloring whose initial node colorings are all the same. Additionally, let $ef^{\mathcal{C}}$ be an arbitrary injective edge filtration. If $G$ and $H$ are distinguishable by WL test, then ${\rm TED}(G, ef^{\mathcal{C}}) \neq {\rm TED}(H, ef^{\mathcal{C}})$. Moreover, there exists a pair of non-isomorphic graphs $(G, H) \in \mathcal{G} \times \mathcal{G}$ such that ${\rm TED}(G, ef^{\mathcal{C}}) \neq {\rm TED}(H, ef^{\mathcal{C}})$ while the WL test cannot distinguish them.  
\end{theorem}


\begin{proof}
    Let $G$ and $H$ be two graphs that are distinguishable by the WL test, and fix an injective edge filtration $ef^{\mathcal{C}}$.  We divide the graph ($G$, $H$) into non-isolated components ($G^+$, $H^+$) and isolated components ($G^0$, $H^0$). More precisely, we have 
    \begin{align*}
        V(G^0)&=\{u \in V(G) \text { $|$ } {\rm deg}(u)=0\}, \\ 
        V(H^0)&=\{v \in V(H) \text { $|$ } {\rm deg}(v)=0\}, \\
        V(G^+)&=\{u \in V(G) \text { $|$ } {\rm deg}(u)>0\}, \\
        V(H^+)&=\{v \in V(H) \text { $|$ } {\rm deg}(v)>0\}. 
    \end{align*}
    Then $V(G)=V(G^0) \sqcup V(G^+)$ and $V(H)=V(H^0) \sqcup V(H^+)$. If $|V(G)| \neq |V(H)|$, then ${\rm ph}_{\rm VR}^{0}(G, ef^{\mathcal{C}}) \neq {\rm ph}_{\rm VR}^{0}(H, ef^{\mathcal{C}})$ clearly holds. So we may assume that $|V(G)|=|V(H)|$. We divide it into two cases: $(i)$ $|V(G^0)| \neq |V(H^0)|$ and $(ii)$ $|V(G^0)| = |V(H^0)|$.  
    
    First, assume the case $(i)$. For a contradiction, suppose not, that is, ${\rm ph}_{\rm VR}^{k}(G, ef^{\mathcal{C}}) = {\rm ph}_{\rm VR}^{k}(H, ef^{\mathcal{C}})$ for all $k=0, 1$. Since $V(G^0)$ and $V(H^0)$ are isolated, the multi-set of all non-essential points (that is, $(0, t)$ where $t \neq \infty$) of ${\rm ph}_{\rm VR}^{0}(G, ef^{\mathcal{C}})$ (and ${\rm ph}_{\rm VR}^{0}(H, ef^{\mathcal{C}})$ respectively) and the multi-set ${\rm ph}_{\rm VR}^{1}(G, ef^{\mathcal{C}})$ (and ${\rm ph}_{\rm VR}^{1}(H, ef^{\mathcal{C}})$ respectively) should come from $G^+$ (and $H^+$ respectively). Hence $E^{\mathcal{C}}(G^+)$ and $E^{\mathcal{C}}(H^+)$ should be the same as multi-sets, which implies that
    \begin{equation} \label{eq:pos_deg_eq}
        \sum_{\substack{u \in V(G^+), \\ \mathcal{C}(u)=C}} {\rm deg}(u) = \sum_{\substack{v \in V(H^+), \\ \mathcal{C}(v)=C}} {\rm deg}(v)
    \end{equation}
    for any color $C$. However, since $|V(G^0)| \neq |V(H^0)|$, we have $|V(G^+)| \neq |V(H^+)|$. Thus there exists a color $C_0$ such that $|\{u \in V(G^+) \text{ $|$ } \mathcal{C}(u)=C_0\}| \neq |\{v \in V(H^+) \text{ $|$ } \mathcal{C}(v)=C_0\}|$. Since the WL test satisfies the degree assumption, all $u \in V(G^+)$ and $v \in V(H^+)$ whose WL color is $C_0$ should have the same positive degree, which we call it $d$. Hence 
    \begin{align*}
        \sum_{\substack{u \in V(G^+), \\ \mathcal{C}(u)=C_0}} {\rm deg}(u) &= d \cdot |\{u \in V(G^+) \text{ $|$ } \mathcal{C}(u)=C_0\}| \\
        &\neq d \cdot |\{v \in V(H^+) \text{ $|$ } \mathcal{C}(v)=C_0\}| = \sum_{\substack{v \in V(H^+), \\ \mathcal{C}(v)=C_0}} {\rm deg}(v),    
    \end{align*}
    which contradicts to Equation \ref{eq:pos_deg_eq}. Hence 
    \begin{equation*}
        ({\rm ph}_{\rm VR}^{0}(G, ef^{\mathcal{C}}), \text{ } {\rm ph}_{\rm VR}^{1}(G, ef^{\mathcal{C}})) \neq ({\rm ph}_{\rm VR}^{0}(H, ef^{\mathcal{C}}), \text{ } {\rm ph}_{\rm VR}^{1}(H, ef^{\mathcal{C}}))
    \end{equation*} 
    for the case $(i)$. 

    Now it remains to prove the case $(ii)$ for the first statement. Since $|V(G^0)| = |V(H^0)|$ and all the isolated nodes of $G$ (and $H$ respectively) correspond to the essential points (that is, $(0, \infty)$) in ${\rm ph}_{\rm VR}^{0}(G, ef^{\mathcal{C}})$ (and ${\rm ph}_{\rm VR}^{0}(H, ef^{\mathcal{C}})$ respectively), ${\rm ph}_{\rm VR}^{k}(G, ef^{\mathcal{C}}) \neq {\rm ph}_{\rm VR}^{k}(H, ef^{\mathcal{C}})$ is equivalent to ${\rm ph}_{\rm VR}^{k}(G^+, ef^{\mathcal{C}}) \neq {\rm ph}_{\rm VR}^{k}(H^+, ef^{\mathcal{C}})$ for any $k=0, 1$. Moreover, since 
    \begin{equation*}
    \{\mathcal{C}(u) \text{ $|$ } u \in V(G), \text{ } {\rm deg}(u)=0\} \cap \{\mathcal{C}(u) \text{ $|$ } u \in V(G), \text{ } {\rm deg}(u)>0\} = \emptyset, 
    \end{equation*}
    the fact that WL test can distinguish $G$ and $H$ implies that it can also distinguish $G^{+}$ and $H^{+}$. In other words, $G^+$ and $H^+$ have no isolated nodes and can be distinguishable by WL test. Since WL test satisfies the degree assumption, Lemma \ref{lem:core} implies that ${\rm ph}_{\rm VR}^{k}(G^+, ef^{\mathcal{C}}) \neq {\rm ph}_{\rm VR}^{k}(H^+, ef^{\mathcal{C}})$ for some $k=0, 1$. Thus $({\rm ph}_{\rm VR}^{0}(G, ef^{\mathcal{C}}), \text{ } {\rm ph}_{\rm VR}^{1}(G, ef^{\mathcal{C}})) \neq ({\rm ph}_{\rm VR}^{0}(H, ef^{\mathcal{C}}), \text{ } {\rm ph}_{\rm VR}^{1}(H, ef^{\mathcal{C}}))$ holds as desired. 

    \begin{figure}[t]
        \centering
        \includegraphics[width=0.3\textwidth]{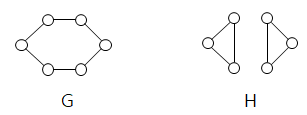}
        \caption{An example of a pair of non-isomorphic graphs that are not distinguishable by WL test.}
        \label{fig:thm_fig}
    \end{figure}

    Next, we will prove the second statement. Let $G$ be the graph on the left of Figure \ref{fig:thm_fig} and $H$ the graph on the right. It is clear that the WL test cannot distinguish them since stable WL colorings of all nodes are identical and the number of nodes of $G$ and $H$ are the same. However, for any injective edge filtration $ef^{\mathcal{C}}$, the number of non-essential elements of ${\rm ph}_{\rm VR}^{0}(G, ef^{\mathcal{C}})$ is 5 while the number of non-essential elements of ${\rm ph}_{\rm VR}^{0}(H, ef^{\mathcal{C}})$ is 4. Thus ${\rm ph}_{\rm VR}^{0}(G, ef^{\mathcal{C}}) \neq {\rm ph}_{\rm VR}^{0}(H, ef^{\mathcal{C}})$, that is, $({\rm ph}_{\rm VR}^{0}(\cdot, ef^{\mathcal{C}}), \text{ } {\rm ph}_{\rm VR}^{1}(\cdot, ef^{\mathcal{C}}))$ can distinguish $G$ and $H$ as desired. 
\end{proof}

\section{Algorithm: Line Graph Vietoris-Rips Persistence Diagram} \label{section:lgvr}

In this section, we propose a novel neural-network-based algorithm, named \emph{Line Graph Vietoris-Rips Persistence Diagram}, to implement TED. First, we construct the map $t_{\phi}$ that transforms a colored graph into a \emph{colored line graph} (Definition~\ref{defn:colored_line_graph}) in Section~\ref{subsection:construction}. Next, through $t_{\phi}$, we propose an algorithm to construct an injective edge filtration, which is the core of our Line Graph Vietoris-Rips Persistence Diagram, in Section~\ref{subsection:lgvr_construction}. Finally, we analyze its theoretical expressivity in Section \ref{subsection:theoretical_LGVR}.

\subsection{Construction of the map $t_{\phi}:(\mathcal{G}, \mathcal{C}) \rightarrow (L_{\mathcal{G}}, \mathcal{C}^{\phi})$} \label{subsection:construction} 

First of all, we briefly recall \emph{line graph} $L_{G}$ of a graph $G \in \mathcal{G}$. In graph theory, a line graph is a type of graph where the vertices correspond to the edges of a given graph $G$, and two vertices in the line graph are connected by an edge if and only if their corresponding edges in $G$ share a common endpoint. Formally, it is defined as follows:

\begin{definition} \label{defn:line_graph}
Let $G=(V(G), E(G)) \in \mathcal{G}$ be a graph. Its line graph $L_G=(V(L_G), E(L_G))$ is a graph such that (1)) each node in $V(L_G)$ represents an edge in $E(G)$, and (2) for any $\{u_1, u_2\} \neq \{v_1, v_2\} \in E(G)$ with $u_1, u_2, v_1, v_2 \in V(G)$, their corresponding nodes in $V(L_G)$ are adjacent if and only if $u_i=v_j$ for some $i=1, 2$ and $j=1, 2$.
\end{definition}

Throughout this paper, we will write the node of the line graph $L_G$ corresponding to $\{\{u,v\}\} \in E(G)$ as $l_{\{\{u,v\}\}}$. Moreover, we denote $L_{\mathcal{G}}:=\{L_{G} \text{ }|\text{ } G \in \mathcal{G}\}$. Now we first propose a \emph{colored line graph}, which is a coloring version of a line graph (Definition~\ref{defn:line_graph}). This is defined using the node coloring $\mathcal{C}$ of a given colored graph $(G, \mathcal{C})$. 

\begin{definition} \label{defn:colored_line_graph}
    Let $(G, \mathcal{C})$ be a graph with node coloring $\mathcal{C}$. A \emph{colored line graph $(L_G, \mathcal{C}^{h})$ of $(G, \mathcal{C})$ with respect to $h$} is the line graph $L_G$ of $G$ with the node coloring $\mathcal{C}^{h}$ such that for any $l_{\{\{u, v\}\}} \in V(L_G)$, $\mathcal{C}^{h}(l_{\{\{u, v\}\}})=h(\{\{\mathcal{C}(u), \mathcal{C}(v)\}\})$, where $l_{\{\{u, v\}\}}$ is a node in $L_G$ corresponding to $\{\{u, v\}\} \in E(G)$ and $h$ is a hash map on $\{\text{ } \{\{\mathcal{C}(u), \mathcal{C}(v)\}\} \in E^{\mathcal{C}}(G) \text{ $|$ } u, v \in V(G), G \in \mathcal{G}\}$.
\end{definition}

In order to elaborate our algorithm, \emph{Line Graph Vietoris-Rips (LGVR) Persistence Diagram}, we first construct a map $t_{\phi}$ that transforms $(G, \mathcal{C}) \in (\mathcal{G}, \mathcal{C})$ into a colored line graph $(L_G, \mathcal{C}^{\phi}) \in (L_{\mathcal{G}}, \mathcal{C}^{\phi})$. Recall that $\chi$ denotes a space of node features of $(\mathcal{G}, \mathcal{C})$ containing $(0,\dots,0)$, where $N \in \mathbb{N}$. Let $\mathcal{M}_{\chi}(2)=\{\{\{x, y\}\} \text{ $|$ } x, y \in \chi\}$, and let $m$ be a multi-layer perceptron with learnable parameters. Now, define a map 
\begin{equation*}
    \phi: \mathcal{M}_{\chi}(2) \rightarrow \mathbb{R}^{2N}, \text{ } \{\{x, y\}\} \mapsto (x+\eta \cdot m(x) + y + \eta \cdot m(y), \text{ }|x+\eta \cdot m(x) - y - \eta \cdot m(y)|),
\end{equation*}
where $\eta$ is a learnable scalar parameter. Then we can define the map $t_{\phi}: (\mathcal{G}, \mathcal{C}) \rightarrow (L_{\mathcal{G}}, \mathcal{C}^{\phi})$ as follows: for a given $(G, \mathcal{C}) \in (\mathcal{G}, \mathcal{C})$, let $l_{\{\{u, v\}\}}$ be the node in $L_{G}$ corresponding to $\{\{u, v\}\} \in E(G)$. Since $\{\{\mathcal{C}(u), \mathcal{C}(v)\}\} \in E^{\mathcal{C}}(G) \subseteq \mathcal{M}_{\chi}(2)$, we can assign a coloring $\phi(\{\{\mathcal{C}(u), \mathcal{C}(v)\}\})$ to a node $l_{\{\{u, v\}\}} \in V(L_{G})$, that is,  
\begin{equation*}
    \mathcal{C}^{\phi}(l_{\{\{u, v\}\}}):=\phi(\{\{\mathcal{C}(u), \mathcal{C}(v)\}\}).
\end{equation*} 
Since each edge in $G$ has a corresponding node in $L_{G}$, $t_{\phi}$ can be well-defined by assigning colors to all the nodes in $L_{G}$ via $\phi$. 

\subsection{Construction of Line Graph Vietoris-Rips (LGVR) Persistence Diagram} \label{subsection:lgvr_construction}

\begin{figure}[t]
    \centering
    \includegraphics[width=0.95\textwidth]{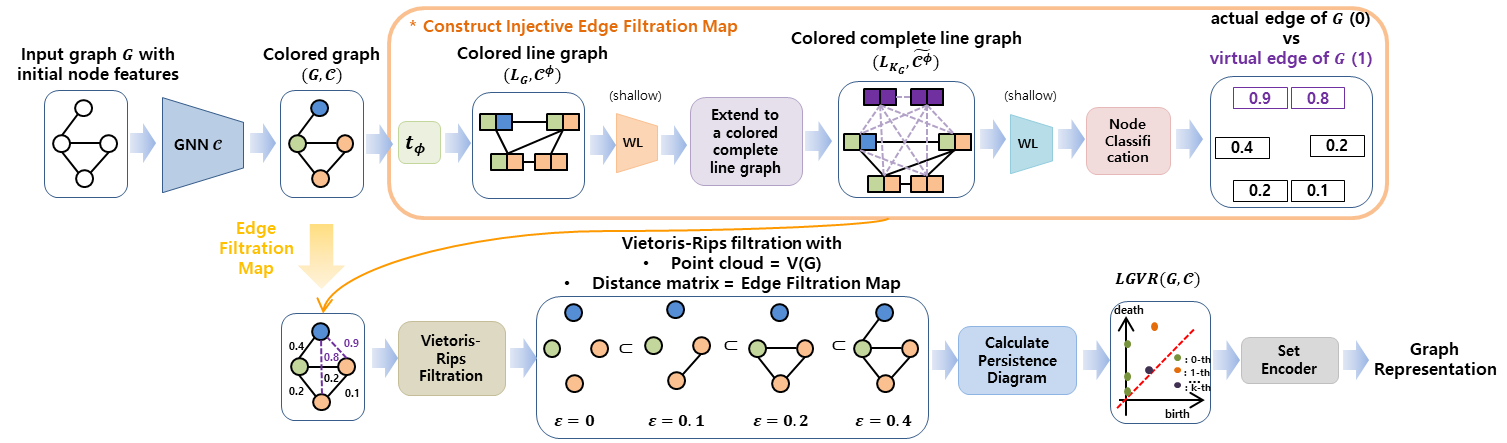}
    \caption{An overall framework of Algorithm~\ref{algo:lgvr_algo}. Note that the map $t_{\phi}$ transforms a colored graph $(G, \mathcal{C})$ into a colored line graph $(L_{G}, \mathcal{C}^{\phi})$ as described in {\bf Construction of the map $t_{\phi}$}.}
    \vspace{-1\baselineskip}
    \label{fig:overall_lgvr}
\end{figure}

\begin{figure}[t]
\begin{algorithm}[H]
\caption{Line Graph Vietoris-Rips (LGVR) Persistence Diagram of $(G, \mathcal{C})$}\label{algo:lgvr_algo}
\begin{footnotesize}
\begin{algorithmic}
    \Input{Labeled graph dataset $(G, y_{G}) \in (\mathcal{G}, y_{\mathcal{G}})$}
    \Initialize{\strut 
    $\mathcal{C}(G|\theta^{\mathcal{C}}) \gets $ GNN with weight $\theta^{\mathcal{C}}$ \\ 
    $t_{\phi}((G, \mathcal{C})|\theta^{t_{\phi}}) \gets $ a map with weight $\theta^{t_{\phi}}$ in  {\bf Construction of the map $t_{\phi}$} \\
    $P_{1}(G|\theta^{P_1}) \gets $ (shallow) node-level GINs with weights $\theta^{P_{1}}$ \\
    $P_{2}(G|\theta^{P_2}) \gets $ (shallow) $[0,1]$-valued node-level GINs with weights $\theta^{P_{2}}$ \\
    $S(\cdot|\theta^{S}) \gets $ a set encoder with weight $\theta^{S}$ \Comment{See Appendix~\ref{subsection:set_encoder}} \\ 
    $m_{\rm task}(\cdot|\theta^{m_{\rm task}}) \gets $ a task-specific MLP with weight $\theta^{m_{\rm task}}$
    }
    \For{T = 1 to \text{total iteration}}
        \State Sample a mini-batch $\{(G_1, y_{G_1}),\dots,(G_n, y_{G_n})\}$ of size $n$ from $(\mathcal{G}, y_{\mathcal{G}})$
        \For{t = 1 to $n$}
            \State $(G_t, \mathcal{C}) \gets \mathcal{C}(G_t | \theta^{\mathcal{C}})$ and construct the edge filtration matrix $A_{G_t}^{\mathcal{C}}$ in {\bf Supplementary}
            \State ${\rm LGVR}(G_t, \mathcal{C}) \gets ({\rm ph}^{0}(\{{\rm VR}_{1}^{\varepsilon}(V(G_t), \text{ } A_{G_t}^{\mathcal{C}})\}_{\varepsilon \in [0, 0.5]}), {\rm ph}^{1}(\{{\rm VR}_{1}^{\varepsilon}(V(G_t), A_{G_t}^{\mathcal{C}})\}_{\varepsilon \in [0, 0.5]}))$
        \EndFor
        \State Update all the weights $\theta^{\mathcal{C}}$, $\theta^{t_{\phi}}$, $\theta^{P_{1}}$, $\theta^{P_{2}}$, $\theta^{S}$, $\theta^{m_{\rm task}}$ by minimizing the loss 
        \begin{gather*}
            \frac{1}{n} \cdot \sum_{t=1}^{n}(\mathcal{L}_{\rm task}(y_{G_t}, m_{\rm task}(S({\rm LGVR}(G_t, \mathcal{C})|\theta^{S})|\theta^{m_{\rm task}})) + \lambda \cdot \mathcal{L}_{\rm LGVR}(G_t)),
        \end{gather*}
        \State where $\lambda$ is a hyperparameter, $\mathcal{L}_{\rm task}$ is the task-specific loss, and 
        \begin{gather*}
            \mathcal{L}_{\rm LGVR}(G_t)= \frac{1}{|V(L_{K_{G_t}})|} \cdot \sum_{l_{\{\{u_x, u_y\}\}} \in V(L_{K_{G_t}})}(\widetilde{\mathcal{C}^{\phi}}(l_{\{\{u_x, u_y\}\}}) - 1_{l_{\{\{u_x, u_y\}\}} \notin V(L_{G_t})})^2.
        \end{gather*}
    \EndFor
\Algphase{Supplementary - Construction of Edge Filtration Matrix $A_{G}^{\mathcal{C}}$ of $(G, \mathcal{C})$ of size $|V(G)| \times |V(G)|$}
    \State $(L_{G}, \mathcal{C}^{\phi}) \gets P_1(t_{\phi}((G, \mathcal{C})|\theta^{t_{\phi}})|\theta^{P_1})$
    \State $K_{G} \gets \text{the complete graph of } V(G)$
    \State $\pi_{G} \gets \text{the natural inclusion map: } V(L_{G})=E(G) \hookrightarrow E(K_{G})=V(L_{K_{G}})$
    \For{$v \in V(L_{K_{G}})$}
        \State ${\widetilde{\mathcal{C}^{\phi}_{0}}(v) \gets
            \begin{cases}
                \mathcal{C}^{\phi}(u) & \text{if $v=\pi_{G}(u)$ for some $u \in V(L_G)$,}  \\
                (0, \dots, 0) & \text{otherwise.}
            \end{cases}
        }$
    \EndFor
    \State Extend $(L_{G}, \mathcal{C}^{\phi})$ to $(L_{K_{G}}, \widetilde{\mathcal{C}^{\phi}})$, and then $(L_{K_{G}}, \widetilde{\mathcal{C}^{\phi}}) \gets P_2((L_{K_{G}}, \widetilde{\mathcal{C}^{\phi}_{0}})|\theta^{P_{2}})$ \Comment{See Appendix~\ref{subsection:extension_to_complete}}
    \For{$i, j \in \{1, \dots, |V(G)|\}$ corresponding to $u_i, u_j \in V(G)$ with $l_{\{\{u_i, u_j\}\}} \in V(L_{K_G})$,}
        \State ${(A_{G}^{\mathcal{C}})_{i,j} \gets
            \begin{cases}
                0 & \text{if $i=j$,}  \\
                \widetilde{\mathcal{C}^{\phi}}(l_{\{\{u_i, u_j\}\}}) & \text{otherwise}
            \end{cases}
        }$
        \Comment{Edge Filtration Map $\widetilde{\mathcal{C}^{\phi}}(\cdot)$ and Matrix $A_{G}^{\mathcal{C}}$ of $(G, \mathcal{C})$}
    \EndFor
\end{algorithmic}
\end{footnotesize}
\end{algorithm}
\vspace{-2\baselineskip}
\end{figure}

In this section, we will elaborate on a neural network-based algorithm, named \emph{Line Graph Vietoris-Rips (LGVR) Persistence Diagram}, with the same expressivity as TED. Pseudocode for the construction of LGVR is described in Algorithm~\ref{algo:lgvr_algo} (See Figure~\ref{fig:overall_lgvr} for the overall framework). Here, we briefly explain it. 

The core of LGVR is to construct an injective edge filtration matrix $A_G^{\mathcal{C}}$ that contains distance information between nodes in a colored graph $(G, \mathcal{C})$ (See orange box in Figure~\ref{fig:overall_lgvr}). To do this, we first convert $(G, \mathcal{C})$ into a colored line graph $(L_G, \mathcal{C}^{\phi})$ through a map $t_{\phi}$. Now that the edge information of $G$ has been converted into the node information of $L_G$, we perform the binary node classification task (actual vs virtual) so that the nodes of $L_G$ corresponding to the actual edges of $G$ have a value of $0$. However, since all nodes in $L_G$ correspond to actual edges in $G$, they are all trained to have values sufficiently close to $0$ during node classification, which can hinder the acquisition of rich information. To address this, we extend $(L_G, \mathcal{C}^{\phi})$ into a colored complete line graph $(L_{K_{G}}, \widetilde{\mathcal{C}^{\phi}})$, and perform the task on $(L_{K_{G}}, \widetilde{\mathcal{C}^{\phi}})$. In this way, we extract meaningful edge information by adding virtual edges to $G$ and training to distinguish actual edges from virtual ones (Appendix~\ref{subsection:extension_to_complete}). Based on the extracted edge information of $G$, we construct the matrix $A_G^{\mathcal{C}}$. Finally, ${\rm LGVR}(G,\mathcal{C})$ is defined as the persistence diagram of $\{{\rm VR}^{\varepsilon}_{1}(V(G), A_{G}^{\mathcal{C}})\}_{\varepsilon \in [0, 0.5]}$ (See Appendix~\ref{subsubsection:range_of_varepsilon} for $\varepsilon \in [0, 0.5]$). Similar to conventional GNNs, LGVR is trained in an end-to-end fashion with task-specific loss $\mathcal{L}_{\rm task}$ (for example, cross entropy) and the LGVR loss $\mathcal{L}_{\rm LGVR}$ in Algorithm~\ref{algo:lgvr_algo}. 

We conclude this section with the complexity analysis of LGVR. Since the complexity of LGVR is dominated by the calculation of persistence homology, we will focus on explaining the computational complexity of dimensions $0$ and $1$. In short, they can be computed efficiently with a worst-case complexity of $\mathcal{O}(m\alpha(m)))$ for a graph with $m$ sorted edges (\cite{comp_top}), where $\alpha(\cdot)$ refers to the inverse Ackermann function, which can be considered practically constant for all purposes. Therefore, the computation of persistence homology is mainly affected by the complexity of sorting all edges, which is $\mathcal{O}(m\log(m)))$.

\subsection{Theoretical Expressivity of LGVR} \label{subsection:theoretical_LGVR}

The most crucial point in LGVR is whether the edge filtration in Algorithm~\ref{algo:lgvr_algo} can be injective. In this regard, we will first state and prove the most essential lemma that demonstrates the injectivity of the edge filtration used in LGVR as follows. 

\begin{lemma} \label{lem:injectivity}
    Assume that $\chi$ is countable, and let $\mathcal{M}_{\chi}(2)=\{\{\{x, y\}\} \text{ $|$ } x, y \in \chi\}$. Then there exists an injective map 
    \begin{equation*}
    \phi: \mathcal{M}_{\chi}(2) \rightarrow \tilde{\chi} \oplus \tilde{\chi} \hookrightarrow \mathbb{R}^{2N},
    \end{equation*}
    where $\tilde{\chi} \hookrightarrow \mathbb{R}^{N}$ and $\tilde{\chi}$ is countable. 
\end{lemma}

\begin{proof}
    Since $\chi$ is countable, it is easy to see that there exists an injective map $f: \chi \rightarrow \mathbb{R}$ and infinitely many $\varepsilon \in \mathbb{R}$ such that for any $d \in \mathbb{N}$, 
    \begin{equation} \label{eq:empty2}
        \Big(\bigcup_{X \in \mathcal{M}_{\chi}(d)}\{\sum_{\substack{x \in X, \\ rsgn(x)=\pm 1}} rsgn(x) \cdot \varepsilon \cdot (f(x), \dots, f(x)) \in \mathbb{R}^{N}\}\Big) \cap \chi = \{\vec{0}\}.
    \end{equation}
    By fixing a desired $f$ and $\varepsilon$, we put $\tilde{\chi}=\{x + \varepsilon \cdot (f(x), \dots, f(x)) \in \mathbb{R}^{N} \text{ $|$ } x \in \chi\}$. It is trivial that $\tilde{\chi}$ is countable. 

    Now consider the map $\phi: \mathcal{M}_{\chi}(2) \rightarrow \tilde{\chi} \oplus \tilde{\chi}$,
    \begin{equation*}
    \{\{x, y\}\} \mapsto (x+\varepsilon \cdot f^{N}(x) + y + \varepsilon \cdot f^{N}(y),\text{ } |x+\varepsilon \cdot f^{N}(x) - y - \varepsilon \cdot f^{N}(y)|),
    \end{equation*}
    where $f^{N}(x)=(f(x), \dots, f(x)) \in \mathbb{R}^{N}$. Note that $\phi$ is well-defined since it is order-invariant. Hence it remains to show that $\phi$ is injective. For the injectivity, assume that
    \begin{align} \label{eq:inj}
        \begin{split}
            (x_1+\varepsilon \cdot f^{N}(x_1) + y_1 + \varepsilon \cdot f^{N}(y_1),&\text{ } |x_1 + \varepsilon \cdot f^{N}(x_1) - y_1 - \varepsilon \cdot f^{N}(y_1)|) \\
            = (x_2+\varepsilon \cdot f^{N}(x_2) + y_2 + \varepsilon \cdot f^{N}(y_2),&\text{ } |x_2+\varepsilon \cdot f^{N}(x_2) - y_2 - \varepsilon \cdot f^{N}(y_2)|)
        \end{split}
    \end{align}
    for some $x_1, x_2, y_1, y_2 \in \chi$. We need to show that $\{\{x_1, y_1\}\}=\{\{x_2, y_2\}\}$ holds under (\ref{eq:inj}). Note that the first equality of (\ref{eq:inj}) implies that $x_1 + y_1 - x_2 - y_2 = \varepsilon \cdot (f^N(x_2) + f^N(y_2) - f^N(x_1) - f^N(y_1))$. By (\ref{eq:empty2}), both sides should be $0$. Hence $f^N(x_1) + f^N(y_1) = f^N(x_2) + f^N(y_2)$, which implies that 
    \begin{equation} \label{eq:add}
        f(x_1)+f(y_1)=f(x_2)+f(y_2).    
    \end{equation} 
    Moreover, the second equality of (\ref{eq:inj}) implies that $x_1 + \varepsilon \cdot f^{N}(x_1) - y_1 - \varepsilon \cdot f^{N}(y_1) = \delta \cdot (x_2+\varepsilon \cdot f^{N}(x_2) - y_2 - \varepsilon \cdot f^{N}(y_2))$, where $\delta=(\delta_1, \dots, \delta_N)$ and 
    \begin{equation*}
        \delta_i=
        \begin{cases}
        1 & \parbox[t]{11cm}{if the signs of $i$-th component of $x_1 + \varepsilon \cdot f^{N}(x_1) - y_1 - \varepsilon \cdot f^{N}(y_1)$ and $x_2+\varepsilon \cdot f^{N}(x_2) - y_2 - \varepsilon \cdot f^{N}(y_2)$ are the same,}  \\
        -1 & \text{otherwise }
        \end{cases}
    \end{equation*}
    for any $i=1, \dots, N$. Again, the same argument as the first equality implies that 
    \begin{equation} \label{eq:subtract}
        f(x_1) - f(y_1) = \delta \cdot (f(x_2) - f(y_2)), \quad \text{that is,} \quad |f(x_1) - f(y_1)| = |f(x_2) - f(y_2)|.
    \end{equation}
    By interchanging $x_i$ and $y_i$, if necessary, we may assume that $f(x_1) \ge f(y_1)$ and $f(x_2) \ge f(y_2)$. Then both equations (\ref{eq:add}) and (\ref{eq:subtract}) imply that $f(x_1)=f(x_2)$ and $f(y_1)=f(y_2)$. Since $f$ is injective, we conclude that $x_1=x_2$ and $y_1=y_2$, that is, $\{\{x_1, y_1\}\}=\{\{x_2, y_2\}\}$, which implies the injectivity of $\phi$ as desired.     
\end{proof}

Lemma~\ref{lem:injectivity} shows that under a countable universe, the map $\phi$ in Section~\ref{subsection:construction} is injective. Moreover, since the subsequent steps also maintain the injectivity, this implies that the edge filtration map $\{\{u_{i},u_{j}\}\} \mapsto \widetilde{\mathcal{C}^{\phi}}(l_{\{\{u_i, u_j\}\}})$ in Algorithm~\ref{algo:lgvr_algo} is inejctive Based on this, we show that LGVR satisfies the essential property (Lemma~\ref{lem:core}, Theorem \ref{thm:wl_vr_comparison}) of TED, which implies that LGVR has the same expressivity as TED.

\begin{theorem} \label{thm:main}
    Assume that $\chi$ is countable. Then the following statements hold:
    \begin{enumerate}
        \item Let $\mathcal{C}$ be an arbitrary node coloring of $\mathcal{G}$ satisfying the degree assumption. For any $G, H \in \mathcal{G}$ with no isolated nodes, if $\{\{\mathcal{C}(v) \text{ $|$ } v \in V(G)\}\} \neq \{\{\mathcal{C}(u) \text{ $|$ } u \in V(H)\}\}$, then ${\rm LGVR}(G, \mathcal{C}) \neq {\rm LGVR}(H, \mathcal{C})$. 
        \item Let $\mathcal{C}$ be the stable WL coloring whose initial node colorings are all the same. If $G, H \in \mathcal{G}$ are distinguishable by $\mathcal{C}$, then we have ${\rm LGVR}(G, \mathcal{C}) \neq {\rm LGVR}(H, \mathcal{C})$. Moreover, there exists a pair of graphs $(G, H) \in \mathcal{G} \times \mathcal{G}$ such that ${\rm LGVR}(G, \mathcal{C}) \neq {\rm LGVR}(H, \mathcal{C})$. 
    \end{enumerate}
\end{theorem}

\begin{proof}
    First, we prove $(1)$. Lemma \ref{lem:injectivity} and the universal approximation theorem of multi-layer perceptrons (\cite{universal1, universal2}) imply that our map 
    \begin{equation*}
        \phi: \mathcal{M}_{\chi}(2) \rightarrow \mathbb{R}^{2N}, \quad \{\{x, y\}\} \mapsto (x+\varepsilon \cdot m_{\mathcal{G}}^{*}(x) + y + \varepsilon \cdot m_{\mathcal{G}}^{*}(y), \text{ }|x+\varepsilon \cdot m_{\mathcal{G}}^{*}(x) - y - \varepsilon \cdot m_{\mathcal{G}}^{*}(y)|)
    \end{equation*} 
    in Section~\ref{subsection:construction} induces the injective edge filtration of $\mathcal{G}$ with respect to $\mathcal{C}$. Moreover, since the WL message passing scheme does not decrease the expressive power in distinguishing node colorings of line graphs, the edge filtration of the LGVR that derives the positive symmetric matrix $A_{G}^{\mathcal{G}}$ is injective. Finally, since $G$ and $H$ are graphs with no isolated nodes and the coloring $\mathcal{C}$ satisfies the degree assumption, the desired result is a direct consequence of Lemma \ref{lem:core}.

    Next assume that $\mathcal{C}$ is the stable WL coloring whose initial node colorings are all the same. As mentioned in the proof above, we know that the edge filtration of the LGVR that derives the positive symmetric matrix $A_{G}^{\mathcal{C}}$ is injective. Hence the statement $(2)$ can be proved similarly as the proof of Theorem~\ref{thm:wl_vr_comparison}. 
\end{proof}

\section{Model Framework} \label{section:model_framework}

From Theorem~\ref{thm:main}, we can infer two important results regarding LGVR: (1) LGVR can preserve arbitrary node coloring information, and (2) especially for WL type coloring, LGVR has stronger expressive powers, which enables the construction of more powerful topological GNNs. In this section, we will delve deeper into the construction of topological GNNs. Before providing a framework for applying LGVR to GNN, we first remark one important aspect of LGVR. From a theoretical perspective, LGVR can guarantee the expressive powers of node colorings due to the injectivity of edge filtration. However, from a practical perspective, if specific node information plays a crucial role in determining the characteristics of a graph due to its rich node coloring information, it is likely that LGVR may not extract a graph representation that properly reflects this since LGVR can only indirectly use the node coloring information by converting it into other (topological) information.

To bridge this gap between theoretical and practical perspectives, we first propose a simple mathematical technique that integrates the expressive powers of both coloring information and topological information induced by LGVR under a countable universe in Section~\ref{subsection:integration}. Depending on the application of this technique, we propose two topological model frameworks in Section~\ref{subsection:model_architectures} and analyze their expressivities in Section~\ref{subsection:model_expressivity}.

\subsection{Integration Technique} \label{subsection:integration}

In this subsection, we will present a general mathematical method for constructing an embedding space that integrates the expressive power of each representation under a countable universe. Proposition \ref{prop:union} and Corollary~\ref{cor:union} propose a framework that preserves all expressive powers of representations by finding $|I|$ integrated embedding maps $f_i$. This integration technique will be used for the construction of one of our model frameworks in Section \ref{subsection:model_architectures} in order to integrate the coloring information of $\mathcal{C}$ and the topological information of LGVR simultaneously. 

First, we state and prove a useful lemma that is crucial to prove Proposition \ref{prop:union}.

\begin{lemma} \label{lem:empty_intersection}
Given $m \in \mathbb{N}$ and countable spaces $\{\chi_{i}\}_{i \in I}$, where $I$ is an index set satisfying $|I| < \infty$, there exists a family of functions $\{f_{i}:\chi_{i} \rightarrow \mathbb{R}^{m}\}_{i \in I}$ so that 
\begin{equation*}
\bigcap_{i \in I}\{f_{i}(x_{i})\text{ }|\text{ }x_{i} \in \chi_{i}\} = \emptyset.
\end{equation*}
\end{lemma}

\begin{proof}
    First, we prove the case when $|I|=2$. We claim that given any functions $f_{1}: \chi_{1} \rightarrow \mathbb{R}^{m}$ and $f_{2}: \chi_{2} \rightarrow \mathbb{R}^{m}$, there exists $\varepsilon \in \mathbb{R}^{m}$ such that $\{f_{1}(x_{1})\text{ }|\text{ }x_{1} \in \chi_{1}\} \cap \{f_{2}(x_{2}) + \varepsilon \text{ }|\text{ }x_{2} \in \chi_{2}\} = \emptyset$. Choose any two functions $f_{1}: \chi_{1} \rightarrow \mathbb{R}^{m}$ and $f_{2}: \chi_{2} \rightarrow \mathbb{R}^{m}$, and consider the set 
    \begin{equation*}
    DF:=\{f_{1}(x_{1}) - f_{2}(x_{2}) \text{ }|\text{ } x_{1} \in \chi_{1} \text{ and } x_{2} \in \chi_{2}\}.
    \end{equation*}
    Since $\chi_{1}$ and $\chi_{2}$ are countable, so is $DF$. Hence there exists infinitely many $\varepsilon \in \mathbb{R}^{m}$ such that $\varepsilon \notin DF$, which proves the claim. Now, by replacing $f_{2}$ by $f_{2}+\varepsilon$ for some $\varepsilon \notin DF$, it is clear that $f_{1}$ and $f_{2}$ satisfy the desired property. 
    
    For general $I$ with $|I|<\infty$, the same arguments implies the existence of infinitely many $\varepsilon_{i} \in \mathbb{R}^{m}$, $i \in I$ satisfying 
    \begin{equation*}
    \bigcap_{i \in I}\{f_{i}(x_{i})+\varepsilon_{i}\text{ }|\text{ }x_{i} \in \chi_{i}\}=\emptyset
    \end{equation*}
    since the countable union of countable sets are countable. Hence by replacing $f_{i}$ by $f_{i}+\varepsilon$, the result holds. 
\end{proof}

Recall that $\mathcal{M}_{\chi}(d)$ is a space of multi-sets of $\chi$ whose cardinality is $d$, and let $\mathcal{M}_{\chi} = \bigcup_{d \in \mathbb{N}} \mathcal{M}_{\chi}(d)$. With this notation in mind, we can prove the following result. 
 
\begin{proposition} \label{prop:union}
    Given $m \in \mathbb{R}$ and countable spaces $\{\chi_{i}\}_{i \in I}$, where $I$ is an index set satisfying $|I| < \infty$, there exists a family of functions $\{f_{i}:\chi_{i} \rightarrow \mathbb{R}^{m}\}_{i \in I}$ so that 
    \begin{equation*}
    g(X_{1}, \dots, X_{|I|})=\sum_{i \in I} \sum_{x \in X_{i}} f_{i}(x)
    \end{equation*}
    is unique for each $(X_{1},\dots,X_{|I|}) \in \mathcal{M}_{\chi_{1}} \times \dots \times \mathcal{M}_{\chi_{|I|}}$ of bounded sizes. 
\end{proposition}

\begin{proof}
    We first prove the case when $|I|=2$. Since $\chi_{1}$ and $\chi_{2}$ are countable, there exists $\tilde{f}_{1}: \chi_{1} \rightarrow \mathbb{R}^{m}$ and $\tilde{f}_{2}: \chi_{2} \rightarrow \mathbb{R}^{m}$ so that $\{\tilde{f}_{1}(x_{1})\text{ }|\text{ }x_{1} \in \chi_{1}\} \cap \{\tilde{f}_{2}(x_{2})\text{ }|\text{ }x_{2} \in \chi_{2}\} = \emptyset$ by Lemma \ref{lem:empty_intersection}. This implies that for each $(X_{1}, X_{2}) \in \mathcal{M}_{\chi_{1}} \times \mathcal{M}_{\chi_{2}}$, we have
    \begin{equation} \label{eq:empty}
    \{\{\tilde{f}_{1}(x_{1}) \text{ }|\text{ } x_{1} \in X_{1}\}\} \cap \{\{\tilde{f}_{2}(x_{2}) \text{ }|\text{ } x_{2} \in X_{2}\}\} = \emptyset.
    \end{equation}
    Let $\chi:=\tilde{f}_{1}(\chi_{1}) \cup \tilde{f}_{2}(\chi_{2})$. Since both $\chi_{1}$ and $\chi_{2}$ are countable, $\chi$ is also countable. Hence \cite[Lemma 5]{gin} implies that there exists a function $f:\chi \rightarrow \mathbb{R}^{n}$ so that $g(X)=\sum_{x \in X} f(x)$ is unique for each $X \in \mathcal{M}_{\chi}$ of bounded size. Now, we put $f_{1} = f \circ \tilde{f}_{1}$ and $f_{2} = f \circ \tilde{f}_{2}$. By Equation (\ref{eq:empty}), every $X \in \mathcal{M}_{\chi}$ of bounded size can be uniquely decomposed by $\{\{\tilde{f}_{1}(x_{1}) \text{ }|\text{ } x_{1} \in X_{1}\}\} \sqcup \{\{\tilde{f}_{2}(x_{2}) \text{ }|\text{ } x_{2} \in X_{2}\}\}$ for some $(X_{1}, X_{2}) \in \mathcal{M}_{\chi_{1}} \times \mathcal{M}_{\chi_{2}}$ of bounded sizes. Hence $g(X_{1}, X_{2})=\sum_{i=1}^{2} \sum_{x \in X_{i}} f_{i}(x)$ is unique for each $(X_{1}, X_{2}) \in \mathcal{M}_{\chi_{1}} \times \mathcal{M}_{\chi_{2}}$ of bounded sizes as desired. 
    
    As in the proof of Lemma \ref{lem:empty_intersection}, the same argument used in $|I|=2$ can be extended to any finite number of countable spaces without difficulty by using the fact that countable union of countable sets is countable. 
\end{proof}

From another point of view, we can reformulate the problem of finding the functions $f_i$ in Proposition \ref{prop:union} into a simpler one of finding $|I|-1$ scalar values. Since this can be easily derived from the proof of Proposition \ref{prop:union}, we will omit the proof.

\begin{corollary} \label{cor:union}
Let $m \in \mathbb{R}$ and let $I$ be an index set satisfying $|I| < \infty$. Given countable spaces $\{\chi_{i}\}_{i \in I}$ and a family of functions $\{f_{i}:\chi_{i} \rightarrow \mathbb{R}^{m}\}_{i \in I}$, there exists infinitely many values $\varepsilon_2, \dots, \varepsilon_{|I|} \in \mathbb{R}$ such that 
\begin{equation*}
g(X_{1}, \dots, X_{|I|})=\sum_{x \in X_{1}}f_1(x) + \sum_{i=2}^{|I|} (1+\varepsilon_{i}) \cdot \sum_{x \in X_{i}}f_i(x)
\end{equation*}
is unique for each $(X_{1},\dots,X_{|I|}) \in \mathcal{M}_{\chi_{1}} \times \dots \times \mathcal{M}_{\chi_{|I|}}$ of bounded sizes. 
\end{corollary}

\subsection{Model Frameworks: {\rm $\mathcal{C}$-LGVR} and {\rm $\mathcal{C}$-LVGR}$^{+}$} \label{subsection:model_architectures}

\begin{figure}[t]
    \centering
    \includegraphics[width=0.95\textwidth]{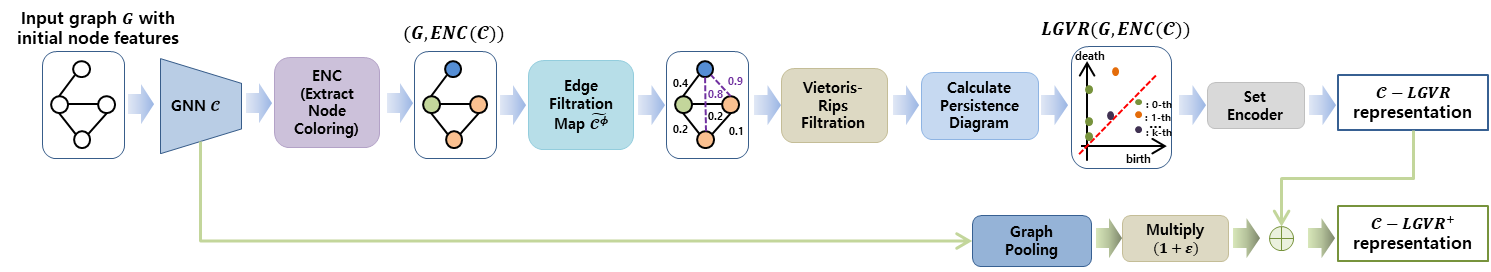}
    \caption{Two model frameworks: $\mathcal{C}$-${\rm LGVR}$ and $\mathcal{C}$-${\rm LGVR}^{+}$. Note that ENC component refers to the process of extracting node coloring from the message passing graph neural network (or coloring) $\mathcal{C}$ as described in Section~\ref{subsection:model_architectures} and \ref{subsection:model_expressivity}.}
    \label{fig:model_fig}
\end{figure}

We propose two model frameworks, {\rm $\mathcal{C}$-LGVR} and {\rm $\mathcal{C}$-LGVR}$^{+}$, depending on the application of the integration technique in Section~\ref{subsection:integration}. For both model frameworks, $\mathcal{C}$ indicates a message passing graph neural network.

\subsubsection{$\mathcal{C}$-LGVR} \label{subsubsection:basic_lgvr}

Note that the output of $\mathcal{C}$ for a graph $G$ may or may not be node coloring, depending on $\mathcal{C}$. For example, if $\mathcal{C}$ is GIN (\cite{gin}), then its output is the node coloring. However, when $\mathcal{C}$ is PPGN (\cite{ppgn}), it provides all node tuple colorings as its output, which is not the node coloring exactly. However, since LGVR relies on the node coloring information of $\mathcal{C}$ as input, we first extract the node colorings from $\mathcal{C}$ in order to bridge this gap. We denote the process of extracting node coloring from $\mathcal{C}$ as ${\rm ENC}(\mathcal{C})$, and we refer to Section~\ref{subsection:model_expressivity} for specific examples. 

After taking node coloring of $G$ from ${\rm ENC}(\mathcal{C})$, we extract the LGVR persistence diagram ${\rm LGVR}(G, \text{ENC}(\mathcal{C}))$ (Algorithm~\ref{algo:lgvr_algo}). Finally, ${\rm LGVR}(G, \text{ENC}(\mathcal{C}))$ is encoded via a set encoder to extract a graph representation (Figure \ref{fig:model_fig}).

\subsubsection{{$\mathcal{C}$-LVGR}$^{+}$}

{\rm $\mathcal{C}$-LVGR}$^{+}$ differs from {\rm $\mathcal{C}$-LVGR} in the way it extracts the graph representation: {\rm $\mathcal{C}$-LVGR}$^{+}$ uses not only the $\mathcal{C}$-LGVR representation but also the pooling output of $\mathcal{C}$ (Figure \ref{fig:model_fig}). Specifically, {\rm $\mathcal{C}$-LGVR}$^{+}$ uses 
\begin{equation} \label{eq:plus}
    (\mathcal{C}\text{-LGVR representation}) + (1+\varepsilon) \cdot (\text{Graph Pooling of }\mathcal{C})
\end{equation} 
as the graph representation vector, where $\varepsilon$ is a learnable parameter in Corollary \ref{cor:union}, to leverage the expressive powers of both representations (See Remark~\ref{rmk:integration_expressivity}). 


\begin{remark} \label{rmk:integration_expressivity}
We will briefly explain how Corollary~\ref{cor:union} is related to the fact that {\rm $\mathcal{C}$-LVGR}$^{+}$ preserves the expressive powers of both coloring information of $\mathcal{C}$ and topological information of LGVR. Note that each component of the output of {\rm $\mathcal{C}$-LVGR}$^{+}$ (Equation~\ref{eq:plus}) can be written simply as follows: given a graph $G$, 
\begin{enumerate}
    \item $(\text{{\rm $\mathcal{C}$-LVGR}$^{+}$ representation}) = MLP(\sum_{x_0 \in {\rm ph}^{0}(G)}S_0(x_0), \sum_{x_1 \in {\rm ph}^{1}(G)}S_1(x_1))$, where ${\rm ph}^{i}(G)$ is the $i$-th persistence diagram of the Vietoris-Rips filtration of $G$ (which is a multi-set) and $S_i$ is a universal set encoder for each $i=0,1$, and
    \item $(\text{Graph Pooling of $\mathcal{C}$})=\sum_{u \in V(G)} \mathcal{C}(u)$.
\end{enumerate}
Since ${\rm ph}^{0}(G)$, ${\rm ph}^{1}(G)$, and $\{\{ \mathcal{C}(u) \text{ }|\text{ } u \in V(G)\}\}$ are all multi-sets of bounded sizes, the universalities of MLP (\cite{universal1, universal2}) and $S_i$ (Appendix~\ref{subsection:set_encoder}) guarantee the existence of $f_i$ in Corollary~\ref{cor:union}. Hence this implies that {\rm $\mathcal{C}$-LVGR}$^{+}$ extracts different representation vectors for each triple of multi-sets $({\rm ph}^{0}(G), {\rm ph}^{1}(G), \{\{ \mathcal{C}(u) \text{ }|\text{ } u \in V(G)\}\})$ by Corollary~\ref{cor:union}. In other words, {\rm $\mathcal{C}$-LVGR}$^{+}$ integrates the coloring information of $\mathcal{C}$, which corresponds to $\{\{ \mathcal{C}(u) \text{ }|\text{ } u \in V(G)\}\}$, and the topological information of LGVR, which corresponds to $({\rm ph}^{0}(G), {\rm ph}^{1}(G))$.
\end{remark}

\subsubsection{Training Loss of $\mathcal{C}$-LGVR and {$\mathcal{C}$-LVGR}$^{+}$}

For both {\rm $\mathcal{C}$-LGVR} and {\rm $\mathcal{C}$-LGVR}$^{+}$, we train them in an end-to-end fashion by minimizing the task-specific loss $\mathcal{L}_{\rm task}$ and the {\rm LGVR} loss $\mathcal{L}_{\rm LGVR}$ in Algorithm~\ref{algo:lgvr_algo} simultaneously, that is, for a set of all learnable parameters $\theta$ and a hyperparameter $\lambda$, 
\begin{equation*}
\theta^{*} = {\rm argmin}_{\theta} \sum_{G \in \mathcal{G}, y_{G} \in \mathcal{Y}_{\mathcal{G}}} \{\mathcal{L}_{\rm task}(G, y_{G}; \theta) + \lambda \cdot \mathcal{L}_{\rm LGVR}(G; \theta)\},
\end{equation*}
where $\mathcal{Y}_{\mathcal{G}}$ is the set of ground truth of $\mathcal{G}$ and $y_{G} \in \mathcal{Y}_{\mathcal{G}}$ is the ground truth of a graph $G \in \mathcal{G}$.

\subsection{Theoretical Expressivities of {\rm GIN-LGVR}, {\rm GIN-LVGR}$^{+}$, and {\rm PPGN-LVGR}$^{+}$} \label{subsection:model_expressivity}

We analyze the expressive powers of three specific models (which will be used in Section~\ref{section:experiments}), {\rm GIN-LGVR}, {\rm GIN-LVGR}$^{+}$, and {\rm PPGN-LVGR}$^{+}$, with the model frameworks in Section \ref{subsection:model_architectures}. As described in Section~\ref{subsubsection:basic_lgvr}, we first introduce the simplest form of ENC($\cdot$) to be used in GIN (\cite{gin}) and PPGN (\cite{ppgn}), respectively (Figure~\ref{fig:model_fig}). 

\textbf{ENC for GIN.} Since GIN provides node colorings as output, ENC is set to the identity: 
\begin{equation*}
    \text{ENC}(Y):=Y \in \mathbb{R}^{n \times d},    
\end{equation*}
where $Y \in \mathbb{R}^{n \times d}$ is the GIN output, $n$ is the number of graph nodes and $d$ is the dimension of node colorings.

\textbf{ENC for PPGN.} Since PPGN provides all node tuple colorings as output, ENC is set up to extract diagonal elements of the node tuple matrix of PPGN: 
\begin{equation*}
    \text{ENC}(Y):=\text{diag}(Y) \in \mathbb{R}^{n \times d},    
\end{equation*}
where $Y \in \mathbb{R}^{n \times n \times d}$ is the PPGN output, $n$ is the number of graph nodes and $d$ is the dimension of node tuple colorings. 

Finally, we will conclude this section with a theoretical analysis of our models as follows.

\begin{corollary} \label{cor:model}
    Assume that $\chi$ is countable. Then the following statements hold:
    \begin{enumerate}
        \item Assume that either all initial node colorings are the same or $\mathcal{G}$ contains no graphs with isolated nodes. Then {\rm GIN-LGVR} is strictly more powerful than WL test. 
        \item {\rm GIN-LVGR}$^{+}$ is strictly more powerful than WL test.
        \item {\rm PPGN-LVGR}$^{+}$ is at least as powerful as $3$-WL test.
    \end{enumerate}
\end{corollary}

\begin{proof}
First $(1)$ is a direct consequence of Theorem \ref{thm:main} and \cite[Theorem 3]{gin}, and $(3)$ follows from Corollary \ref{cor:union} and Remark~\ref{rmk:integration_expressivity}. Hence it remains to show $(2)$. First, Corollary \ref{cor:union} and Remark~\ref{rmk:integration_expressivity} again imply that {\rm GIN-LVGR}$^{+}$ is at least as powerful as the WL test. To prove the 'strictly powerful' part, we need to show that there exists a pair of graphs $G$ and $H$ that {\rm GIN-LVGR}$^{+}$ can distinguish but the WL test cannot. Again, it is easy to see that $G$ and $H$ in Figure~\ref{fig:thm_fig} work as a desired example, which concludes the proof. 
\end{proof}

\section{Experiments} \label{section:experiments}

In this section, we evaluate the performance of our models, {\rm GIN-LGVR}, {\rm GIN-LVGR}$^{+}$, and {\rm PPGN-LVGR}$^{+}$, on several graph classification and regression benchmark datasets. We measure the performance improvements of our models compared to the baseline message passing GNNs $\mathcal{C}$, {\rm GIN} and {\rm PPGN}, to demonstrate that our edge filtration-based approach helps $\mathcal{C}$ to achieve a substantial gain in predictive performance. Furthermore, we compare our models with node filtration-based ones (\cite{gfl}) to validate the superiority of edge filtration over node filtration. In short, we focus on experimental verification of the following claims:
\begin{enumerate}[font={\bfseries},label={Claim \arabic*.}, leftmargin=0.8in]
    \item Superior performances of our model frameworks, {\rm $\mathcal{C}$-LGVR} and {\rm $\mathcal{C}$-LVGR}$^{+}$, that outperform the message passing GNN $\mathcal{C}$.
    \item Superiority of edge filtration-based approach over node filtration-based one.
    \item Experimental validation of our integration technique (Section~\ref{subsection:integration}; Corollary~\ref{cor:union}) that integrates the pooling information of $\mathcal{C}$ and the LGVR information.
\end{enumerate}

\subsection{Experimental Setup}

We conduct experiments on three models: {\rm GIN-LGVR}, {\rm GIN-LVGR}$^{+}$, and {\rm PPGN-LVGR}$^{+}$. Both networks, GIN (\cite{gin}) and PPGN (\cite{ppgn}), are constructed in their most basic form: they both consist of three message passing layers, with a hidden dimension of $64$ for GIN and $400$ for PPGN. To minimize the impacts of other techniques, we refrain from using the jumping knowledge network scheme (\cite{jkn}). Finally, to encode LGVR diagrams, which are a tuple of multi-sets, we use a set encoder. To sufficiently leverage their topological information, we set our set encoder by combining Deep Set (\cite{deepset}) and Set Transformer (\cite{settransformer}) (See Appendix~\ref{subsection:set_encoder} for details). We run all experiments on a single DGX-A100 GPU. Our code is publicly available at \url{https://github.com/samsungsds-research-papers/LGVR}.

\subsection{Datasets}

We evaluate our methods on two different tasks: graph classification and graph regression. For classification, we test our method on 7 benchmark graph datasets: 5 bioinformatics datasets (MUTAG, PTC, PROTEINS, NCI1, NCI109) that represent chemical compounds or protein substructures, and two social network datasets (IMDB-B, IMDB-M) (\cite{benchmark}). For the regression task, we experiment on a standard graph benchmark QM9 dataset (\cite{qm9-1, qm9-4, qm9-2}). It is made up of 134k small organic molecules of varying sizes from 4 to 29 atoms, and the task is to predict 12 real-valued physical quantities for each molecule graph. Further details can be found in Appendix~\ref{app:dataset}.

\subsection{Baseline and Comparison Models} 

To argue the superiority of our method, we essentially compare our models {\rm $\mathcal{C}$-LGVR} and {\rm $\mathcal{C}$-LVGR}$^{+}$ to the same message passing graph neural network $\mathcal{C}$: we compare {\rm GIN-LGVR} and {\rm GIN-LVGR}$^{+}$ to GIN, and {\rm PPGN-LVGR}$^{+}$ to PPGN. To demonstrate the superiority of edge filtration over node filtration, we further compare the performance of both methods for GIN. For fairness in comparison, both models are constructed with the same GIN architecture and extract topological information in a single scale. 

Specifically, we conduct experiments using (1) GFL (\cite{gfl}), which is a single-scale version of node filtration, and (2) GFL$^{+}$ (which we call) that integrates graph pooling and GFL by using the integration technique in Section~\ref{subsection:integration}. All hyperparameters were carried out on the same set based on $\mathcal{C}$ (Appendix~\ref{subsection:hyperparameter}):  the learning rate is set to $\{5 \cdot 10^{-3}, 10^{-3}, 5 \cdot 10^{-4}, 10^{-4}, 5 \cdot 10^{-5}\}$ for GIN while it is set to $\{10^{-4}, 5 \cdot 10^{-5}\}$ for PPGN. The decay rate is set to $\{0.5, 0.75, 1.0\}$ with Adam optimizer (\cite{adam}), and we implement all the models by tuning hyperparameters based on the validation score. 

\subsection{Graph Classification Results} \label{subsection:classification}

We test our models on datasets from the domains of bioinformatics and social networks. Since there is no separate test set for these datasets, for a fair comparison, we follow the standard 10-fold cross-validation based on the same split used by \cite{dgcnn} and report the results according to the evaluation protocol described by \cite{ppgn}: 
\begin{equation} \label{eq:final_perf}
    \max_{i \in \{1,\dots, t\}} \frac{1}{10} \cdot \sum_{k=1}^{10} P_{k,i},
\end{equation}
where $t$ is the total epoch and $P_{k,i}$ is the $k$-fold validation accuracy at $i$-th epoch. 

\subsubsection{Performance Analysis} \label{subsubsection:perf_analysis_class}

Table \ref{tab:classification} presents the performances of our models ({\rm GIN-LGVR}, {\rm GIN-LVGR}$^{+}$, {\rm PPGN-LVGR}$^{+}$) and comparison models ({\rm GIN}, {\rm GFL}, {\rm GFL}$^{+}$, {\rm PPGN}). First of all, we experimentally verify \textbf{Claim 1, 2, 3}. As shown in Table~\ref{tab:classification}, our models show the best performances across all datasets for the message passing GNN $\mathcal{C}$. This empirically demonstrates that our approach helps the GNN $\mathcal{C}$ to achieve substantial gains in predictive performance, which validates \textbf{Claim 1}. Moreover, we found that our approach based on edge filtration shows superior performance compared to node filtration-based methods ({\rm GFL}, {\rm GFL}$^{+}$) by utilizing both nodes and edges to extract more informative representations. This demonstrates the superiority of our edge filtration over node filtration (Figure~\ref{fig:perf_improv}), which validates \textbf{Claim 2}. Finally, we remark that {\rm GIN-LVGR}$^{+}$ generally shows better performances than {\rm GIN-LGVR} and outperforms {\rm GIN} on all datasets: {\rm GIN-LVGR}$^{+}$ achieves the best performance on 5 out of 7 datasets. This empirically demonstrates that our theoretical framework (Corollary~\ref{cor:union}) which integrates the pooling output of $\mathcal{C}$ and the LGVR output without losing information, works well in practice, which validates \textbf{Claim 3}. 

In addition to \textbf{Claim 1, 2, 3}, Table~\ref{tab:classification} provides an interesting finding: the improvement percentages of our models compared to ${\rm GIN}$ are inversely proportional to the amount of initial node information. Specifically, as shown in Figure~\ref{fig:perf_improv}, our models show relatively larger improvement percentages on datasets such as MUTAG, PROTEINS, and IMDB with less initial node information (see the feature column in Table~\ref{tab:dataset_classification}), compared to PTC, NCI1, and NCI109. This is because having an initial node with rich information enables GNNs to extract a sufficiently expressive graph representation using only the node information so that the performance improvement of additional topological information from LGVR decreases. Therefore, this implies that our LGVR is more effective for datasets with less initial node information, which are relatively difficult to analyze.

\begin{table*}[t] 
	\centering
	\caption{Graph classification results (with mean accuracy)}
	\label{tab:classification}
    \scalebox{0.78}{
    \begin{tabular}{|c|c| c c c c c c c |}
		\hline
  		& Model / Dataset & MUTAG  & PTC   & PROTEINS  & NCI1  & NCI109  & IMDB-B & IMDB-M    \\ \hline
  		
        \multirow{5}{*}{$\mathcal{C}$={\rm GIN}} & GIN & 78.89     & 59.71     & 68.11      & 70.19 & 69.34 &  73.8 & 43.8       \\ 
        & {GFL} &  86.11    &  60.0    &   73.06    &   71.14    &  70.53  &   68.7  & 44.67  \\    
        & {GFL}$^{+}$ &  79.44    &   59.11   & 69.55    &  71.16  &  70.12  & 71.9   &  44.33 \\    
        & \textbf{{GIN-LGVR}} & \bf{86.67}   &  60.29    & \bf{73.42}      &   69.22 & 69.39 &  72.5  &  45.47    \\     
        & \textbf{{GIN-LGVR}$^{+}$} &  85.0    &  \bf{61.76}    &  69.55     &  \bf{71.61} & \bf{70.68} &  \bf{74.0} &  \bf{45.8}  \\ \hline
        
        \multirow{2}{*}{$\mathcal{C}$={\rm PPGN}} & PPGN & 88.88     & 64.7     & 76.39      & 81.21 & 81.77 & 72.2  & 44.73    \\ 
         & \textbf{{PPGN-LGVR}$^{+}$} & \bf{91.11}     & \bf{66.47}     & \bf{76.76}      & \bf{83.04} & \bf{81.88} & \bf{73.5}  & \bf{51.0}        \\ \hline
	\end{tabular}
	}
\end{table*}

\begin{figure}[t]
    \centering
    \includegraphics[height=0.3\textwidth, width=0.7\textwidth]{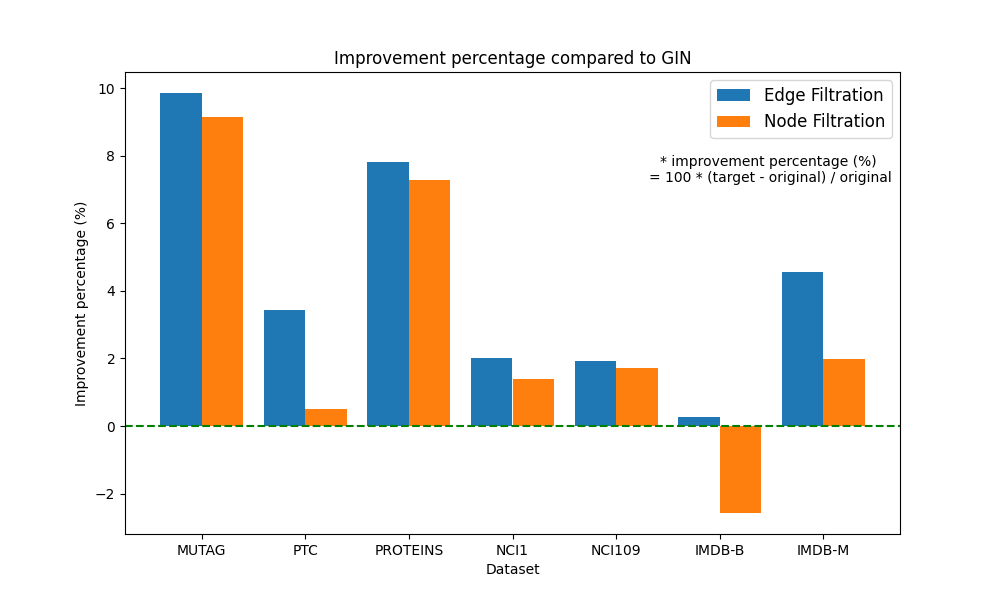}
    \caption{Improvement percentages of two filtrations compared to GIN. The performances of two filtrations are calculated as $\max\{P_{\text{GIN-LGVR}} \text{, } P_{\text{GIN-LVGR}^{+}}\}$ for edge filtration, and $\max\{P_{\text{GFL}} \text{, } P_{\text{GFL}^{+}}\}$ for node filtration, where $P_{\bullet}$ is the model performance of Table~\ref{tab:classification}.}
    \label{fig:perf_improv}
\end{figure}

For a comprehensive analysis, we further analyze the classification results with $\mathcal{C}=\text{GIN}$ from two aspects in the next subsection. Here we will briefly explain them. First, we found that model performances varied significantly by the data split (Table~\ref{tab:classification_stdev}, Figure~\ref{fig:stdev_fig}). Hence we further analyze the standard deviation of the $10$-fold performances in Table~\ref{tab:classification} to determine how stable each model is in training regardless of the data split. We found that our models, {\rm GIN-LGVR} and {\rm GIN-LVGR}$^{+}$, show low standard deviations compared to {\rm GIN}, {\rm GFL}, and {\rm GFL}$^{+}$ (Figure~\ref{fig:stdev_fig}). This implies that {\rm GIN-LVGR} and {\rm GIN-LVGR}$^{+}$, which use both nodes and edges to reflect various characteristics in graph representations, would have enabled more stable learning compared to {\rm GIN}, {\rm GFL}, and {\rm GFL}$^{+}$ that only use nodes. 

Next, we compare the models using performance metric different from Equation~\ref{eq:final_perf}. Metrics are indicators to determine the perspective from which models are compared, hence we measure performances based on another commonly used metric (Equation~\ref{eq:max_perf}) to compare models from a different perspective. Concerning this metric, our models, {\rm GIN-LGVR} and {\rm GIN-LVGR}$^{+}$, still exhibit the best performance compared to others (See Table~\ref{tab:max_perf_classification}).

\subsubsection{Additional Analysis on Graph Classification with $\mathcal{C}=\text{GIN}$} \label{subsection:additional_analysis_experiment}

Here, we will provide additional analysis on the 10-fold classification results of each dataset for five GIN type models in Section~\ref{subsubsection:perf_analysis_class}: {\rm GIN}, {\rm GFL}, ${\rm GFL}^{+}$, {\rm GIN-LVGR}, and {\rm GIN-LVGR}$^{+}$. Before conducting the analysis, we fix some notation. Let $i_{0}$ be the epoch that computes the performance (Equation~\ref{eq:final_perf}), that is, $\text{Equation~(\ref{eq:final_perf})}=\frac{1}{10} \cdot \sum_{k=1}^{10} P_{k,i_0}$. For a notational convenience, we denote $P_{k,i_0}$ by $P_{k}$.

\smallskip
\textbf{Analysis of $P_{k}$ for each model.} In Section~\ref{subsubsection:perf_analysis_class}, we use the average of $\{P_k\}_{k=1,\dots,10}$ as the final performance metric. However, as can be seen in Table~\ref{tab:classification_stdev}, the performance $P_{k}$ varies significantly for each fold $k$. Therefore, we measure the standard deviation of the $10$-fold performances in Table~\ref{tab:classification} to determine how stable each model ({\rm GIN}, {\rm GFL}, ${\rm GFL}^{+}$, {\rm GIN-LVGR}, and {\rm GIN-LVGR}$^{+}$) is in training regardless of the data split.

The standard deviation results for $\{P_k\}_{k=1,\dots,10}$ for each model and dataset are summarized in Figure~\ref{fig:stdev_fig}. Here, we find that {\rm GIN-LVGR} and {\rm GIN-LVGR}$^{+}$ show relatively low standard deviations compared to other models. More specifically, each model shows the following average of the overall standard deviation: $4.42$ (for {\rm GIN-LVGR}), $4.91$ (for {\rm GIN-LVGR}$^{+}$), $5.08$ (for {\rm GIN}), $5.16$ (for {\rm GFL}), and $6.66$ (for {\rm GFL}$^{+}$). Since graphs are composed of nodes and edges, the importance of nodes and edges can vary depending on the data. From this perspective, we claim that these results stem from the characteristics of {\rm GIN-LVGR} and {\rm GIN-LVGR}$^{+}$, which leverage both nodes and edges. In other words, this result implies that {\rm GIN-LVGR} and {\rm GIN-LVGR}$^{+}$, which leverage both nodes and edges to reflect various features in graph representations, would have enabled more stable learning compared to other models ({\rm GIN}, {\rm GFL}, and {\rm GFL}$^{+}$) that rely solely on nodes.

\begin{table*}[t] 
	\centering
	\caption{Minimum and maximum values of $P_k$ for each model and dataset.}
	\label{tab:classification_stdev}
    \scalebox{0.8}{
    \begin{tabular}{|c|c| c c c c c |}
		\hline
  	Dataset / Model	&  & {\rm GIN}  & {\rm GFL}   & {\rm GFL}$^{+}$  & {\rm GIN-LVGR} & {\rm GIN-LVGR}$^{+}$ \\ \hline
        \multirow{2}{*}{MUTAG} & min &   61.11  &  72.22   &  66.67   & 77.78    &   72.22    \\ 
        & max &  94.44   &  100.0   &  100.0   &  100.0   &   100.0    \\  \hline
        
        \multirow{2}{*}{PTC} & min &  50.0   &  47.06   &  44.12   &  44.12   &  55.88    \\ 
        & max &  85.35   &  70.59   &  76.47   &   70.59  &   73.53    \\  \hline

        \multirow{2}{*}{PROTEINS} & min &  59.46   &  63.06   &  58.56   &  63.96   &  58.56     \\ 
        & max &   76.58  &  81.08   &  77.48   &  85.59   &   82.88     \\  \hline

        \multirow{2}{*}{NCI1} & min &  66.91   &  66.67   &  67.15   & 65.94    &   67.15    \\ 
        & max &  74.21   &  74.69   &  74.69   &  71.29   &  74.7     \\  \hline

        \multirow{2}{*}{NCI109} & min &  66.02   &  64.8   &  67.23   &  65.78   &  67.48     \\ 
        & max &  72.09   &  74.51   &  74.02   &  73.54   &  75.97     \\  \hline

        \multirow{2}{*}{IMDB-B} & min &  65.0   &  64.0   &  61.0   &  69.0   &  72.0    \\ 
        & max &  79.0   &  72.0   &   81.0  &  82.0   &   79.0    \\  \hline

        \multirow{2}{*}{IMDB-M} & min &  39.33   &  39.33   &  32.67   &  42.67   &   40.0    \\ 
        & max &  50.0   &  49.33   &   56.67  &  48.0   &   54.0    \\  \hline
	\end{tabular}
	}
\end{table*}

\begin{figure}[t]
    \centering
    \includegraphics[height=0.4\textwidth]{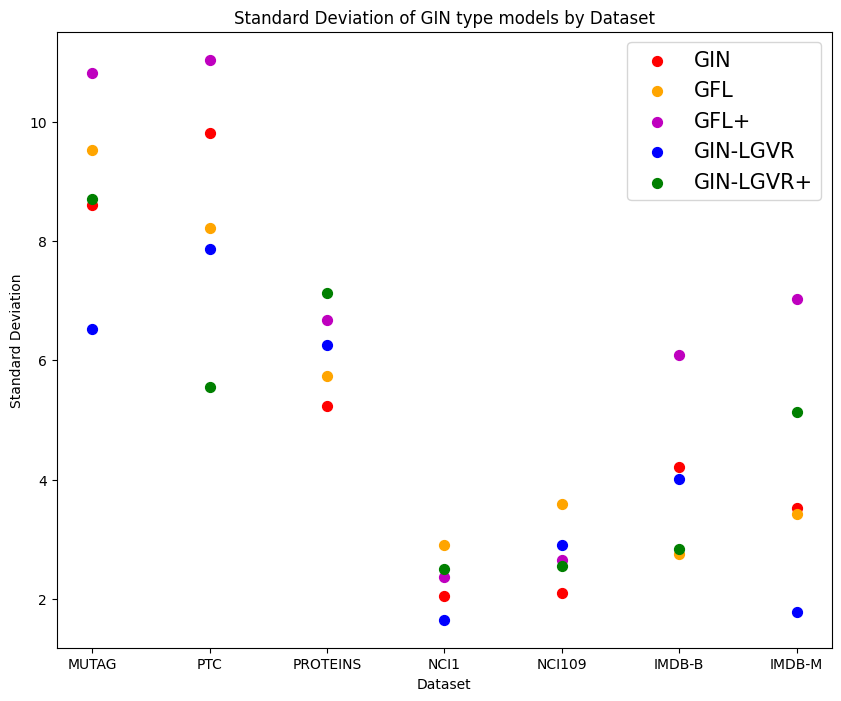}
    \includegraphics[height=0.4\textwidth]{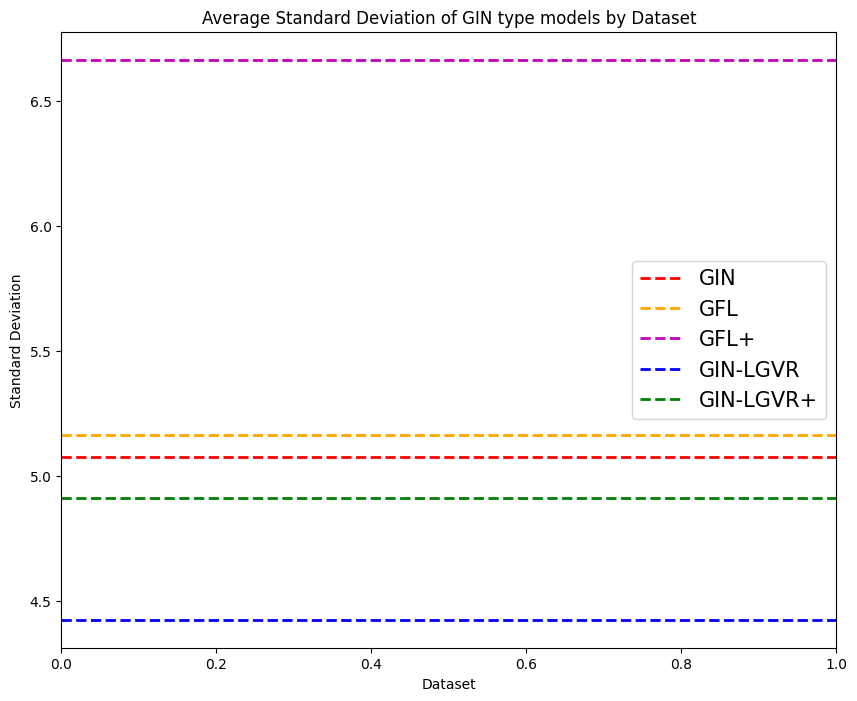}
    \caption{Standard deviations of GIN type models: {\rm GIN}, {\rm GFL}, {\rm GFL}$^{+}$, {\rm GIN-LVGR}, and {\rm GIN-LVGR}$^{+}$. (Left) Each point represents the standard deviation of $10$-fold performances (Table~\ref{tab:classification}) for each model and dataset. (Right) The horizontal dashed line represents the average of the overall standard deviation for each model.}
    \label{fig:stdev_fig}
\end{figure}

\smallskip
\textbf{Other Performance Metric: Mean of the maximum accuracy per fold. } As mentioned in Section~\ref{subsubsection:perf_analysis_class}, since there was no separate test set in graph classification benchmarks, we performed $10$-fold cross-validation and used Equation~\ref{eq:final_perf} as our performance metric. The reason why we use this metric is that it is best suited to evaluate model performance through cross-validation. Specifically, the key aspect of cross-validation is to provide performance estimation and hyperparameters (for example, learning rate, decay rate, epoch, etc) without a test set. However, to reduce the hyperparameter search space, the epoch is typically set to a fixed value, which may result in reduced reliability of performance estimation due to inappropriate epoch settings. To address this, we used Equation~\ref{eq:final_perf} as our final performance metric. Equation~\ref{eq:final_perf} calculates performance estimation for all epochs, preventing a decrease in performance reliability due to inappropriate epoch settings and providing the optimal epoch, which is consistent with the purpose of cross-validation.

However, in addition to Equation~\ref{eq:final_perf}, the average of the maximum validation performance for each fold is also frequently used: 
\begin{equation} \label{eq:max_perf}
    \frac{1}{10} \cdot \sum_{k=1}^{10} (\max_{i \in \{1,\dots, t\}} P_{k,i}),
\end{equation}
where $t$ is the total epochs and $P_{k,i}$ is the $k$-fold validation accuracy at $i$-th epoch. Therefore, we also provide the classification results measured by Equation~\ref{eq:max_perf} in Table~\ref{tab:max_perf_classification}. 

Note that there is a significant difference between the values of Table~\ref{tab:classification} with Equation~\ref{eq:final_perf} and those of Table~\ref{tab:max_perf_classification} with Equation~\ref{eq:max_perf}. Moreover, there is a slight variation in the rankings between models for some datasets. In particular, ${\rm GFL}^{+}$ showed a significant ranking improvement for some datasets (MUTAG, IMDB-B, IMDB-M, for example) when measured by Equation~\ref{eq:max_perf}, due to its high variations as can be seen in Table~\ref{tab:classification_stdev} and Figure~\ref{fig:stdev_fig}. However, there is no change in the model that shows the best performance for each dataset. In other words, even when using Equation~\ref{eq:max_perf} as the performance metric, we confirm that our models, {\rm GIN-LVGR} and {\rm GIN-LVGR}$^{+}$, still show the best performance for all datasets.

\begin{table*}[t] 
	\centering
	\caption{Graph classification results (with Equation~\ref{eq:max_perf})}
	\label{tab:max_perf_classification}
    \scalebox{0.78}{
    \begin{tabular}{|c| c c c c c c c |}
		\hline
  	Model / Dataset & MUTAG  & PTC   & PROTEINS  & NCI1  & NCI109 & IMDB-B & IMDB-M    \\ \hline	
        GIN & 88.33     & 70.29     & 74.5      & 72.31        &  71.97  & 77.3 & 50.0       \\ 
        {GFL} &  92.22    &  71.47    &   76.84    & 73.23    & 72.88  & 75.2  & 49.6  \\    
        {GFL}$^{+}$ &  91.11    &   69.41   & 73.96    &   73.28    & 72.62  & 77.4   &  51.0 \\    
        \textbf{{GIN-LGVR}} & \bf{92.78}   &  70.29    & \bf{77.3}      &   72.34      &  71.77   &  76.7  &  50.07    \\     
        \textbf{{GIN-LGVR}$^{+}$} & 91.11    &  \bf{72.64}    &  74.68     &  \bf{73.94}       &   \bf{73.45}  &  \bf{77.8} &  \bf{52.0}  \\ \hline
	\end{tabular}
	}
\end{table*}

\subsection{Graph Regression Results} \label{subsection:regression}

\begin{table*}[t]
    \centering
    \caption{Graph regression results on the QM9 dataset (with mean absolute error).}
    \label{tab:regression}
    \scalebox{0.7}{
    \begin{tabular}{|c| ccccc|cc|}
        \hline
          & \multicolumn{5}{c|}{$\mathcal{C}$={\rm GIN}} & \multicolumn{2}{c|}{$\mathcal{C}$={\rm PPGN}} \\ \hline 
        Target / Model & GIN & {GFL} & {GFL}$^{+}$ & \textbf{{GIN-LGVR}} & \textbf{{GIN-LGVR}$^{+}$} & PPGN & \textbf{{PPGN-LGVR}$^{+}$} \\
        \hline
        $\mu$    & 0.729 & 0.917 & \bf{0.655} & 1.058 & 0.661 & 0.231 & \bf{0.09} \\
        $\alpha$ & 3.435 & 4.818 & 2.985  & 3.542 & \bf{2.76}  & 0.382 & \bf{0.19} \\
        $\epsilon_{homo}$ & 0.00628 & 0.01595 & \bf{0.00581} & 0.01528 & 0.00683 & 0.00276 & \bf{0.00178} \\
        $\epsilon_{lumo} $& 0.00957 & 0.02356 & 0.00987 & 0.02952 & \bf{0.00911} & 0.00287 & \bf{0.0019} \\
        $\Delta_\epsilon $& 0.01023 & 0.01737 & 0.01003 & 0.03253 & \bf{0.00931} & 0.00406  & \bf{0.00253} \\
        $\langle R^2 \rangle$ & 124.05 & 175.23 & 121.33 & 174.96   & \bf{113.1} & 16.07 & \bf{3.47} \\
        $ZPVE  $& 0.00719 & 0.02354 & 0.00393 & 0.00534 & \bf{0.0026}  & 0.00064  &  \bf{0.00032} \\
        $U_0    $& 17.477  & 18.938 & 18.121 & 17.458  & \bf{15.705}  & 0.234 & \bf{0.216} \\
        $U    $& 17.477  & 18.881 & 18.12 & 17.807  &  \bf{15.706} & 0.234 & \bf{0.215} \\
        $H     $& 17.476 & 18.923 & 18.118  & 17.626  & \bf{15.705} & 0.229 & \bf{0.217} \\
        $G     $& 17.477 & 18.889 & 18.121  & 17.72  & \bf{15.708} & 0.238 & \bf{0.216} \\
        $C_v  $& 1.361  & 3.515 & 1.258 & 1.993  &  \bf{1.163} & 0.184 & \bf{0.081} \\
        \hline
    \end{tabular}
    }
\end{table*}

We evaluate our models on the QM9 dataset, which involves predicting 12 numeric quantities for a given molecular graph. The dataset is split into 80\% train, 10\% validation, and 10\% test. Finally, we use the same network from the classification experiments and compare the performances of models by training a single network to predict all 12 quantities simultaneously. 


Table \ref{tab:regression} presents the test mean absolute error for both our models ({\rm GIN-LGVR}, {\rm GIN-LVGR}$^{+}$, {\rm PPGN-LVGR}$^{+}$) and comparison models ({\rm GIN}, {\rm GFL}, {\rm GFL}$^{+}$, {\rm PPGN}). These results provide several interesting findings. First, {\rm GIN-LVGR}$^{+}$ and {\rm PPGN-LVGR}$^{+}$, show the best performances for almost all 12 quantities, which empirically demonstrates the superior performance of our model for regression task as well. This validates \textbf{Claim 1}. Next, regardless of filtration types, both {\rm GFL} and {\rm GIN-LGVR}, which do not use pooling information, show a decrease in performance compared to {\rm GIN}. As mentioned in the classification experiment, we remark that such decreases in performances stem from the fact that the initial node features of molecule graphs in QM9 are highly informative (Appendix~\ref{app:qm9_dataset}). However, we found that {\rm GIN-LVGR}$^{+}$ outperforms {\rm GIN} and {\rm GIN-LVGR} for all 12 quantities. This empirically validates the effectiveness of our integration framework technique (Section~\ref{subsection:integration}; Corollary~\ref{cor:union}), which validates \textbf{Claim 3}. Finally, to verify \textbf{Claim 2}, we compare our models with node filtration methodologies for the GIN type and show that {\rm GIN-LVGR}$^{+}$ outperforms node filtration-based models ({\rm GFL}, {\rm GFL}$^{+}$) on almost all tasks ($10$ out of $12$ tasks). Through these results, we empirically demonstrate the superiority of our edge filtration-based approach over node filtration-based one for regression tasks as well (\textbf{Claim 2}).

\section{Conclusion} \label{section:conclusion}

We propose a novel edge filtration-based persistence diagram, named Topological Edge Diagram (TED), which can incorporate topological information into any message passing graph neural networks. We mathematically prove that TED can preserve the node embedding information as well as contain additional topological information. We further prove that TED can even strictly increase the expressivity of the WL test. To implement our theoretical foundation, we propose a novel neural network-based algorithm, called Line Graph Vietoris-Rips (LGVR) Persistence Diagram and prove that LGVR has the same expressivity as TED. To evaluate the performance of LGVR, we propose two model frameworks ({$\mathcal{C}$-LVGR} and {$\mathcal{C}$-LVGR}$^{+}$) that can be applied to any message passing GNNs $\mathcal{C}$, and prove that they are strictly more powerful than the WL test. Through our model frameworks, we empirically demonstrate the superior performances of our approach. The downside is that since the ENC only extracts node coloring information, applying the {$\mathcal{C}$-LVGR} would result in losing all non-node coloring information. Although we address this problem by introducing the integration technique in Section~\ref{subsection:integration}, we believe that investigating ENC to prevent this information loss would be an interesting future work.

{\vskip 0.3in\noindent{\large\bf Acknowledgments}\vskip 0.2in\noindent}{The authors would like to thank the Action Editor and anonymous reviewers for their careful reading of the paper and useful suggestions to help improve the exposition of the paper.}




\appendix

\begin{section} {Purpose of extension to $(L_{K_G}, \widetilde{\mathcal{C}^{\phi}})$ in Algorithm~\ref{algo:lgvr_algo}} \label{subsection:extension_to_complete}

In this section, we will explain the reason why we extend a colored line graph $(L_G, \mathcal{C}^{\phi})$ to a colored complete line graph $(L_{K_G}, \widetilde{\mathcal{C}^{\phi}})$. In short, it is done to extract rich topological information about graph $G$ by performing binary node classification on the line graph to distinguish actual edges and virtual edges of $G$.

Suppose that we perform node classification on $(L_G, \mathcal{C}^{\phi})$ instead of $(L_{K_G}, \widetilde{\mathcal{C}^{\phi}})$ with the loss $\mathcal{L}_{\rm LGVR}$ in Algorithm~\ref{algo:lgvr_algo}. Then all the nodes will end up being trained to have a value of $0$ since all the nodes in $L_G$ correspond to actual edges in $G$. This means that all edge values of $G$ will be sufficiently close to $0$ (that is, they are all similar), which can act as a bottleneck in obtaining rich topological information, as the persistence of all homology classes of ${\rm ph}_{\rm VR}^{i}(G, ef^{\mathcal{C}})$ becomes excessively short. To avoid this issue, we add virtual edges to $G$ using the complete graph $K_G$ and perform the node classification on $(L_{K_G}, \widetilde{\mathcal{C}^{\phi}})$ instead of $(L_G, \mathcal{C}^{\phi})$.

\end{section}

\begin{section} {A remark on the range $[0, 0.5]$ of $\varepsilon$ for $\{{\rm VR}^{\varepsilon}_{1}(V(G), A_{G}^{\mathcal{C}})\}_{\varepsilon \in [0, 0.5]}$} \label{subsubsection:range_of_varepsilon}
In this section, we explain why the range of $\varepsilon$ is set to $[0, 0.5]$ in the definition of ${\rm LGVR}(G, \mathcal{C})$: $({\rm ph}^{0}(\{{\rm VR}_{k}^{\varepsilon}(V(G), \text{ } A_{G}^{\mathcal{C}})\}_{\varepsilon \in [0, 0.5]}),\dots, {\rm ph}^{k}(\{{\rm VR}_{k}^{\varepsilon}(V(G), A_{G}^{\mathcal{C}})\}_{\varepsilon \in [0, 0.5]}))$. Since we assign scalar values to edges through $A_{G}^{\mathcal{C}}$ and use them as distances between nodes, the range of $\varepsilon$ determines which edges are significant and used to extract the topological information of the graph. Hence, to set an appropriate range for $\varepsilon$, it is necessary to interpret the values of the distance matrix $A_{G}^{\mathcal{C}}$, which contains information about the edges. The values in $A_{G}^{\mathcal{C}}$ range from 0 to 1, and meaningful (or actual) edges have a value close to 0, while meaningless (or virtual) edges have a value close to 1. Therefore, we determine that meaningful edges have values of 0.5 or less, so we use this to extract the topological information of the graph by setting the range of $\varepsilon$ to $[0, 0.5]$.
\end{section}

\begin{section} {Our Set Encoder} \label{subsection:set_encoder}

To extract the representation of LGVR Diagram, we use both Deep Set (\cite{deepset}) and Set Transformer (\cite{settransformer}). Deep Set has a simple architecture and has been theoretically proven to be a universal approximator of set functions under a countable universe (\cite[Theorem 2]{deepset}). However, Deep Set operates independently on each element of the set, which leads to the disadvantage of discarding all information related to interactions between elements. Inspired by (\cite{transformer}), Lee \etal \cite{settransformer} proposed the Set Transformer, to reflect higher-order interactions between elements in a set. 

Our set encoder is constructed by combining these two models: Deep Set and Set Transformer. Briefly speaking, we first apply Deep Set and then pass the results to Set Transformer. We expect to reflect better element-wise interactions within the set by passing higher-quality feature vectors through the Deep Set to Set Transformer.

\begin{figure}[t]
    \centering
    \includegraphics[width=1.0\textwidth]{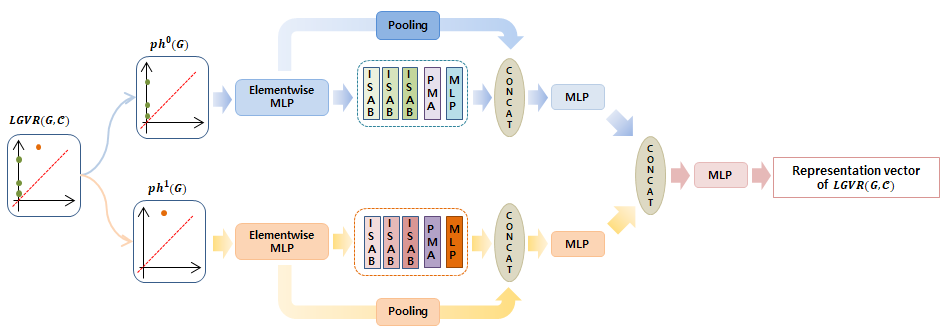}
    \caption{An architecture of our set encoder. Note that Elementwise MLP means performing MLP on each element in the multi-set, and ISAB and PMA respectively denote Induced Set Attention Block and Pooling by Multihead Attention (\cite{settransformer}). After passing the 0-th and 1-st persistence diagrams through Elementwise MLPs, their outputs are not only used as inputs for Set Transformers (\cite{settransformer}) but also used to extract the representation vectors of Deep Set (\cite{deepset}) through pooling (for example, sum). Through this skip connection scheme, our set encoder takes on a form that combines Deep Set and Set Transformer.}
    \label{fig:set_encoder}
\end{figure}

Here we will describe our set encoder architecture in detail. For notational convenience, let 
\begin{equation*}
{\rm ph}^{i}(G):={\rm ph}^{i}(\{{\rm VR}_{1}^{\varepsilon}(V(G), \text{ } A_{G}^{\mathcal{C}})\}_{\varepsilon \in [0, 0.5]})
\end{equation*} 
for $i=0, 1$. Since multi-sets $ph^0(G)$ and $ph^1(G)$ are independent of each other, we construct each set encoder and concatenate them. First, we construct the element-wise encoding part of Deep Set using multi-layer perceptrons ${\rm MLP}^{0}_{0}$ and ${\rm MLP}^{0}_{1}$ with two layers: for $i=0, 1$, 
\begin{equation*}
{\rm DE}_i(G):= {\rm MLP}^{0}_{i}({\rm ph}^{i}(G)),
\end{equation*}
where 
\begin{align*}
{\rm ph}^{i}(G)&=\{\{x_1,\dots, x_g\}\} \text{, and} \\
{\rm MLP}_{i}({\rm ph}^{i}(G))={\rm MLP}_{i}(\{\{&x_1,\dots, x_g\}\}):=\{\{{\rm MLP}_{i}(x_1), \dots, {\rm MLP}_{i}(x_g)\}\}
\end{align*}


Next, we construct Set Transformers ${\rm ST}_0$ and ${\rm ST}_1$ taking multi-sets ${\rm DE}_0(G)$ and ${\rm DE}_1(G)$ as their inputs. For each $i=0,1$, the encoder of each ${\rm ST}_i$ is composed of three layers of Induced Set Attention Block (ISAB) ${\rm ISAB}_{i}^{j}$, $j=0,1,2$, and the decoder is composed of a single layer of Pooling by Multi-head Attention (PMA) ${\rm PMA}_i$ and fully connected layer ${\rm FC}_i$: for each $i=0,1$, 
\begin{align*}
    (\text{Encoder of ST}_{i})({\rm DE}_i(G))&:={\rm ISAB}_{i}^{2}({\rm ISAB}_{i}^{1}({\rm ISAB}_{i}^{0}({\rm DE}_i(G)))), \\
    {\rm ST}_{i}({\rm DE}_i(G))&:={\rm FC}_{i}({\rm PMA}_{i}(\text{Encoder of ST}_{i})({\rm DE}_i(G))).
\end{align*}
Now, for each $i=0,1$, we define a new vector ${\rm DST}_{i}(G)$ by concatenating the result of Deep Set, $\sum_{x \in {\rm DE}_{i}(G)} x$, with the result of Set Transformer, ${\rm ST}_{i}({\rm DE}_i(G))$, and then passing it through an additional multi-layer perceptrons ${\rm MLP}^{1}_{i}$: for each $i=0, 1$, 
\begin{equation*}
{\rm DST}_{i}(G):={\rm MLP}^{1}_{i}([\sum_{x \in {\rm DE}_{i}(G)} x \text{ $|$ } {\rm ST}_{i}({\rm DE}_i(G))]). 
\end{equation*}
Finally, we extract the embedding vector for the {\rm LGVR} diagram by concatenating ${\rm DST}_{0}(G)$ and ${\rm DST}_{1}(G)$, and passing it through the last multi-layer perceptrons ${\rm MLP}^{2}$:
\begin{equation*}
{\rm SetEncoder}(G):={\rm MLP}^{2}([{\rm DST}_{0}(G) \text{ $|$ } {\rm DST}_{1}(G)]). 
\end{equation*}
By combining the Deep Set and Set Transformer as described above, we construct our set encoder. Note that Figure \ref{fig:set_encoder} depicts the overall architecture of our set encoder. 

\end{section}

\begin{section}{Experimental Details}

In this section, we explain some details of datasets and hyperparameter settings used in the experiments (Section~\ref{section:experiments}).

\begin{subsection}{Details of Datasets} \label{app:dataset}

We give detailed descriptions of datasets used in our experiments:  MUTAG, PTC, PROTEINS, NCI1, NCI109, IMDB-B, IMDB-M, and QM9. 


\begin{table}[t] 
	\centering
    \caption{Details of datasets for graph classification task. Note that features (5th column) refer to the number of classes of the initial node feature. When the feature is NA, it indicates that there is no initial node feature so the degree is used as the initial node feature. Moreover, \#classes (6th column) means the number of classes of labels.}
    \label{tab:dataset_classification}
    \scalebox{0.85}{
        \begin{tabular}{|l| c c c c c|}
    	    \hline
      		Dataset & \#graphs  & avg. \#nodes & avg. \#edges & features & \#classes \\ \hline
      		MUTAG & 188 & 17.9 & 19.79 & 7 & 2 \\
      		PTC & 344 & 14.29  & 14.69 & 22 & 2 \\
      		PROTEINS & 1113 & 39.06  & 72.82 & 3 & 2 \\
      		NCI1 & 4110 & 29.87  & 32.3 & 37 & 2 \\
      		NCI109 & 4127 & 4 & 29.6 & 38 & 2 \\
                IMDB-B & 1000 & 19.77  & 96.53 & NA & 2 \\ 
      		IMDB-M & 1500 & 13  & 65.94 & NA & 3 \\ \hline
        \end{tabular}
    }
\end{table}

\begin{subsubsection}{Bioinformatics datasets} \label{app:bio}

In bioinformatic graphs, the nodes have categorical input features, and we test 5 datasets in our experiments: MUTAG, PTC, PROTEINS, NCI1, and NCI109. MUTAG is a dataset of 188 mutagenic aromatic and heteroaromatic nitro compounds. PTC is consisted of 344 chemical compounds that reports the carcinogenicity for male and female rats. Moreover, PROTEINS is a dataset whose nodes are secondary structure elements and there is an edge between two nodes if they are neighborhoods in the amino-acid sequence or in 3D space. Finally, NCI1 and NCI109, made publicly available by the National Cancer Institute, are two subsets of balanced datasets of chemical compounds screened for ability to suppress or inhibit the growth of a panel of human tumor cell lines. Statistics for these datasets are summarized in Table \ref{tab:dataset_classification}. 

\end{subsubsection}

\begin{subsubsection}{Social network datasets} \label{app:social}

In social network graphs, no features are provided for nodes so we set all node features to be the same (thus, node features are uninformative). We test two social network datasets in our experiments: IMDB-B (binary) and IMDB-M (multi-class). Both datasets are movie collaboration datasets. Each graph provides information about actors/actresses and genres of different movies on IMDB. For each graph, nodes correspond to actors/actresses and an edge is drawn between actors/actresses who appear in the same movie. Each graph corresponds to a specific genre label, and the task is to classify which genre a given graph belongs to. IMDB-B consists of collaboration graphs on \textit{Action} and \textit{Romance} genres, and IMDB-M is a multi-class version of IMDB-B derived from \textit{Comedy}, \textit{Romance}, and \textit{Sci-Fi} genres. Statistics for these datasets are summarized in Table \ref{tab:dataset_classification}. 

\end{subsubsection}

\begin{subsubsection}{QM9 dataset} \label{app:qm9_dataset}

QM9 is the dataset consisting of 134k small organic molecules of varying sizes from 4 to 29 atoms. Each graph is represented by an adjacency matrix and input node features, which can be obtained from the pytorch-geometric library \cite{qm9-3}. Note that the input node features are of dimension $18$, which contain information about the distance between atoms, categorical data on the edges, etc.

This dataset has three characteristics: (1) nodes correspond to atoms and edges correspond to close atom pairs, (2) edges are purely distance-based, and (3) it only provides the coordinates of atoms and atomic numbers as node features. Moreover, the number of classes of initial node features is $18$, and the task is to predict 12 real-valued physical quantities for each molecule graph. We provide a brief description of each regression target in Table \ref{tab:qm9_target}. A more detailed explanation of QM9 dataset can be found in \cite{qm9-5}.

\begin{table*}[t]
    \centering
    \caption{A brief description of each regression target on QM9 dataset.}
    \label{tab:qm9_target}
    \scalebox{0.7}{
    \begin{tabular}{|l| c|}
        \hline
        Target & Description \\
        \hline
        $\mu$    &  Dipole moment  \\
        $\alpha$ & Isotropic polarizability \\
        $\epsilon_{homo}$ & Highest occupied molecular orbital energy \\
        $\epsilon_{lumo}$ & Lowest unoccupied molecular orbital energy \\
        $\Delta_\epsilon$ & Gap between $\epsilon_{homo}$ and $\epsilon_{lumo}$ \\
        $\langle R^2 \rangle$ &  Electronic spatial extent \\
        $ZPVE$ &  Zero point vibrational energy \\
        $U_0$ &  Internal energy at 0K \\
        $U$ &  Internal energy at 298.15K  \\
        $H$ &  Enthalpy at 298.15K  \\
        $G$ &   Free energy at 298.15K   \\
        $C_v$ &  Heat capacity at 298.15K   \\
        \hline
    \end{tabular}
    }
\end{table*}

\end{subsubsection}

\end{subsection}

\begin{subsection} {Details of Hyperparameter Settings} \label{subsection:hyperparameter}

The hyperparameters we tune for each dataset are: (1) learning rate (LR), (2) decay rate (DR), (3) batch size (BS), and (4) the number of epochs (Ep). For classification, the search space for each hyperparameter is as follows: when $\mathcal{C}$ is GIN, the learning rate is set to $\{10^{-3}, 5 \cdot 10^{-4}, 10^{-4}\}$ and when $\mathcal{C}$ is PPGN, they are set to $\{10^{-4}, 5 \cdot 10^{-5}\}$. Moreover, for regression, in the case of PPGN, it maintains the same setting as classification, whereas in the case of GIN, the learning rate is set to $\{5 \cdot 10^{-3}, 10^{-3}, 5 \cdot 10^{-4}, 10^{-4}, 5 \cdot 10^{-5}\}$. In both cases, we use the Adam optimizer (\cite{adam}) and decay the learning rate by $\{0.5, 0.75, 1.0\}$ every $20$ epochs. 

Among the search space for hyperparameter, we summarize the hyperparameter settings for each dataset used in our experiments in Table \ref{tab:hyperparameter_gin} and \ref{tab:hyperparameter_ppgn}.  

\begin{table}[t] 
	\centering
    \caption{GIN Type: Details of hyperparameter settings used in our experiments for graph classification and regression tasks.}
    \label{tab:hyperparameter_gin}
	\scalebox{0.58}{
    \begin{tabular}{|l| llll| llll| llll| llll| llll|}
	    \hline 
            & \multicolumn{4}{c|}{{\rm GIN}} & \multicolumn{4}{c|}{{\rm GIN-LVGR}} & \multicolumn{4}{c|}{{\rm GIN-LVGR}$^{+}$} & \multicolumn{4}{c|}{{\rm GFL}} & \multicolumn{4}{c|}{{\rm GFL}$^{+}$} \\ \hline
            Dataset & LR & DR & BS & Ep & LR & DR & BS & Ep & LR & DR & BS & Ep & LR & DR & BS & Ep & LR & DR & BS & Ep \\ \hline
  		MUTAG & $10^{-3}$ & 0.75 & 5 & 500 & $5*10^{-4}$ & 0.75 & 5 & 500 & $5*10^{-4}$  & 0.75 & 5 & 500 & $5*10^{-4}$ & 0.75 & 5 & 500 &  $10^{-3}$ & 0.75 & 5 & 500 \\
  		PTC & $10^{-3}$ & 0.75 & 5 & 400 & $10^{-3}$  & 0.75 & 5 & 400 &  $5*10^{-4}$  & 0.75 & 5 & 400 &  $5*10^{-4}$ & 0.75 & 5 & 400 &  $5*10^{-4}$  & 0.75 & 5 & 400\\
  		PROTEINS & $10^{-3}$ & 0.75 & 5 & 400 & $10^{-3}$ & 0.75 & 5 & 400 &  $10^{-3}$  & 0.75 & 5 & 400 & $5*10^{-4}$  & 0.75 & 5 & 400 & $10^{-3}$  & 0.75 & 5 & 400 \\
  		NCI1 & $5*10^{-4}$ & 0.75 & 5 & 200 &  $5*10^{-4}$ & 0.75 & 5 & 200 & $5*10^{-4}$  & 0.75 & 5 & 200 & $10^{-3}$  & 0.75 & 5 & 200 & $10^{-3}$  & 0.75 & 5 & 200 \\
            NCI109 & $5*10^{-4}$  & 0.75 & 5 & 250 & $5*10^{-4}$  & 0.75 & 5 & 250 & 
        $10^{-3}$ & 0.75 & 5 & 250  &  $5*10^{-4}$  & 0.75 & 5 & 250 & $5*10^{-4}$  & 0.75 & 5 & 250  \\
  		IMDB-B & $10^{-3}$ & 0.75 & 5 & 150 & $10^{-4}$ & 0.75 & 5 & 150 &  $10^{-3}$   & 0.75 & 5 & 150 &  $10^{-3}$ & 0.75 & 5 & 150 & $10^{-3}$  & 0.75 & 5 & 150   \\ 
  		IMDB-M & $10^{-3}$ & 0.75 & 5 & 150 & $10^{-4}$ & 0.75 & 5 & 150 &  $10^{-3}$   & 0.75 & 5 & 150 &  $10^{-3}$ & 0.75 & 5 & 150 & $10^{-3}$  & 0.75 & 5 & 150   \\ 
  		QM9 & $5*10^{-3}$ & 0.8 & 64 & 300 & $5*10^{-5}$ & 0.8 & 64 & 300 &  $10^{-3}$  & 0.8 & 64 & 300 & $5*10^{-4}$ & 0.8 & 64 & 300 & $5 * 10^{-3}$ & 0.8 & 64 & 300  \\ \hline
    \end{tabular}
    }
\end{table}

\begin{table}[t] 
	\centering
    \caption{PPGN Type: Details of hyperparameter settings used in our experiments for graph classification and regression tasks.}
    \label{tab:hyperparameter_ppgn}
    \scalebox{0.7}{
	\begin{tabular}{|l| llll| llll|}
	    \hline 
            & \multicolumn{4}{c|}{{\rm PPGN}} & \multicolumn{4}{c|}{{\rm PPGN-LVGR}$^{+}$} \\ \hline
            Dataset & LR & DR & BS & Ep & LR & DR & BS & Ep \\ \hline
  		MUTAG & $10^{-4}$ & 1.0 & 5 & 500 & $10^{-4}$ & 1.0 & 5 & 500 \\
  		PTC & $10^{-4}$ & 1.0 & 5 & 400 & $5*10^{-5}$ & 1.0 & 5 & 400 \\
  		PROTEINS & $10^{-3}$ & 0.5 & 5 & 400 & $5*10^{-5}$ & 0.5 & 5 & 400 \\
  		NCI1 & $10^{-4}$ & 0.75 & 5 & 200 & $5*10^{-5}$ & 0.75 & 5 & 200 \\
            NCI109 & $10^{-4}$ & 0.75 & 5 & 250 & $5*10^{-5}$ & 0.75 & 5 & 250 \\
  		IMDB-B & $5 * 10^{-5}$ & 0.75 & 5 & 150 & $5*10^{-5}$ & 0.75 & 5 & 150 \\ 
  		IMDB-M & $10^{-4}$ & 0.75 & 5 & 150 & $5*10^{-5}$ & 0.75 & 5 & 150 \\ 
  		QM9 & $10^{-4}$ & 0.8 & 64 & 300 & $10^{-4}$ & 0.8 & 64 & 300 \\ \hline
    \end{tabular}
    }
\end{table}

\end{subsection}

\end{section}

\vskip 0.2in
\bibliography{./reference}

\end{document}